\documentclass{article}

\usepackage[nonatbib,final]{neurips_2024}

\usepackage{fdsymbol}

\DeclareFontFamily{U}{FdSymbolA}{}
\DeclareFontShape{U}{FdSymbolA}{m}{n}{
    <-> s * [1] FdSymbolA-Book
}{}
\DeclareFontShape{U}{FdSymbolA}{m}{b}{
    <-> s * [1] FdSymbolA-Medium
}{}
\DeclareSymbolFont{fdsymbols}{U}{FdSymbolA}{m}{n}
\SetSymbolFont{fdsymbols}{bold}{U}{FdSymbolA}{m}{b}

\DeclareMathSymbol{\medtriangleright}{\mathbin}{fdsymbols}{86}
\DeclareMathSymbol{\medtriangleup}{\mathbin}{fdsymbols}{87}
\DeclareMathSymbol{\medtriangleleft}{\mathbin}{fdsymbols}{88}
\DeclareMathSymbol{\medtriangledown}{\mathbin}{fdsymbols}{89}
\usepackage{rotating}
\usepackage{longtable}
\usepackage{xcolor}         % colors
\usepackage{booktabs}
\usepackage{multirow}
\usepackage{enumitem}
\usepackage{wrapfig}
\usepackage[tableposition=bottom]{caption}
\usepackage{subcaption}
\usepackage{float}
\usepackage[linesnumbered,boxed,ruled,commentsnumbered]{algorithm2e}
\usepackage{multirow, booktabs}
\usepackage{amsmath}
\usepackage{amsfonts}
\usepackage{amsthm}
\usepackage[most]{tcolorbox}
\usepackage{wrapfig}
\usepackage{utfsym}

\usepackage[pagebackref,breaklinks,colorlinks,bookmarks=false,urlcolor=black,citecolor={green!60!black}]{hyperref}

\usepackage{booktabs}

\usepackage{xspace}

\newcommand{\heart}{$\;\!$\usym{2665}}

\theoremstyle{plain}
\newtheorem{theorem}{Theorem}[section]

\theoremstyle{definition}
\newtheorem{definition}[theorem]{Definition}

\theoremstyle{remark}

\newcommand{\ms}[2]{{#1\tiny{$\pm$#2}}}

\newcommand{\ums}[2]{{\underline{#1}\tiny{$\pm$#2}}}
\newcommand{\ubms}[2]{{\underline{\textbf{#1}}\tiny{$\pm$#2}}}

\newcommand{\um}[1]{{\underline{#1}}}
\newcommand{\ubm}[1]{{\underline{\textbf{#1}}}}

\newcommand{\figref}[1]{Figure~\ref{#1}}%
\newcommand{\tabref}[1]{Table~\ref{#1}}%
\renewcommand{\eqref}[1]{Eq.~(\ref{#1})}

\usepackage[textsize=tiny]{todonotes}
\makeatletter
\DeclareRobustCommand\onedot{\futurelet\@let@token\mathbfv@onedotaux}
\def\mathbfv@onedotaux{\ifx\@let@token.\else.\null\fi\xspace}
\def\eg{\emph{e.g}\onedot} 
\def\ie{\emph{i.e}\onedot} 
 
 \def\vs{\emph{vs}\onedot}
\def\wrt{w.r.t\onedot}

\def\etal{\emph{et al}\onedot}

\def\bd{{\color{blue}{$\vardiamondsuit$}}}
\def\rh{{\color{red}{$\varheartsuit$}}}
\def\cc{{\color{cyan}{$\clubsuit$}}}
\def\os{{\color{orange}{$\spadesuit$}}}

\newcommand{\yifei}[1]{{\color{black}#1}}

\definecolor{beaublue}{rgb}{0.86274,0.90980,0.98431}

\title{Understanding and Mitigating Hyperbolic Dimensional Collapse in Graph Contrastive Learning}

\author{%
Yifei Zhang$^{\ddagger}$ Hao Zhu$^{\heartsuit}$ Menglin Yang$^{\dagger}$ Jiahong Liu$^{\ddagger}$ Rex Ying$^{\dagger}$ Irwin King$^{\ast,\,\ddagger}$ Piotr Koniusz\thanks{Corresponding authors.$\;\;$PK also in charge of math/theorems.}$\;^{\,,\heartsuit,\,\S}$\\\\[-5pt]
$^{\ddagger}$Chinese University of Hong Kong $\;\;$
$^{\dagger}$Yale University \\
$^\S$Australian National University $\;\;$
$^{\heartsuit}$Data61$\!${\color{red}\heart}CSIRO\\\\[-5pt]
{\tt\small \{yfzhang, jhliu, king\}@cse.cuhk.edu.hk, menglin.yang@outlook.com, rex.ying@yale.edu,}\\
{\tt\small \{allen.zhu, piotr.koniusz\}@data61.csiro.au}
}

\begin{document}

\maketitle

\vspace{-0.6cm}
\begin{abstract}
Learning generalizable self-supervised graph representations for downstream tasks is challenging. To this end, Contrastive Learning (CL) has emerged as a leading approach. The embeddings of CL are arranged on a hypersphere where similarity is measured by the cosine distance. However, many real-world graphs, especially of hierarchical nature, cannot be embedded well in the Euclidean space. Although the hyperbolic embedding is suitable for hierarchical representation learning, naively applying CL to the hyperbolic space may result in the so-called dimension collapse, \ie, features will concentrate mostly within few density regions, leading to poor utilization of the whole feature space. Thus, we propose a novel contrastive learning framework to learn high-quality graph embeddings in hyperbolic space. Specifically, we design the alignment metric that effectively captures the hierarchical data-invariant information, as well as we propose a substitute of the uniformity metric to prevent the so-called dimensional collapse. We show that in the hyperbolic space one has to address the leaf- and height-level uniformity related to properties of trees. In the ambient space of the hyperbolic manifold these notions translate into imposing an isotropic ring density towards boundaries of Poincaré ball. Our experiments support the efficacy of our method. 
\end{abstract}

\section{Introduction}
\setlength{\columnsep}{14pt}%

\begin{wrapfigure}{r}{0.4\linewidth}
\vspace{-2.0cm}
\begin{subfigure}[t]{0.49\linewidth}
\includegraphics[trim={0.2cm 0.2cm 0.2cm 0.2cm}, clip=true, width=\linewidth]{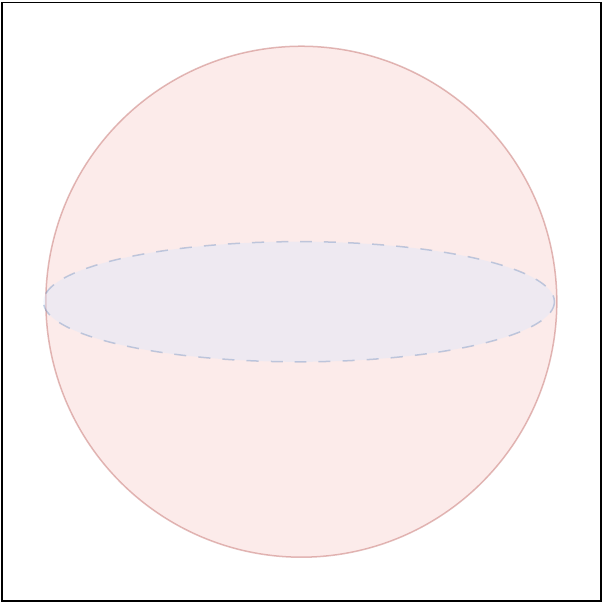}
\end{subfigure}
\begin{subfigure}[t]{0.49\linewidth}
\includegraphics[trim={0.2cm 0.2cm 0.2cm 0.2cm}, clip=true, width=\linewidth]{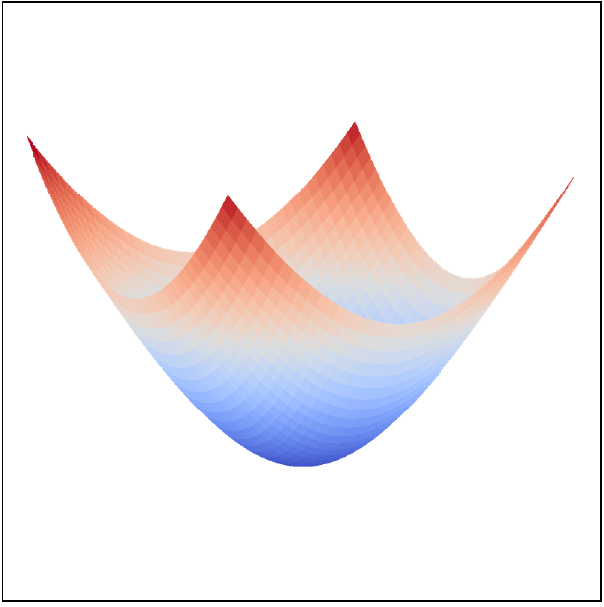}
\end{subfigure}
\vspace{-0.7cm}
\caption{Hypersphere \vs. Hyperbola (viewed in the ambient space)\label{fig:sphere_vs_hyperboola}.}
\vspace{0.2cm}
\begin{subfigure}[t]{0.49\linewidth}
\includegraphics[trim={0.2cm 0.2cm 0.2cm 0.2cm}, clip=true, width=\linewidth]{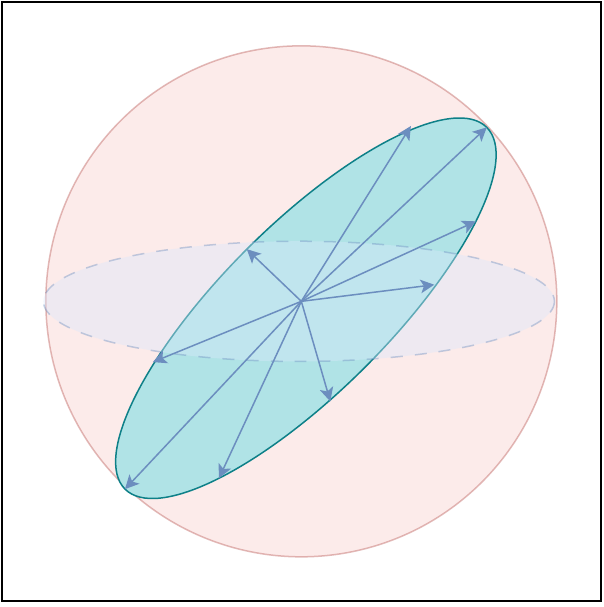}
\end{subfigure}
\begin{subfigure}[t]{0.49\linewidth}
\includegraphics[trim={0.2cm 0.2cm 0.2cm 0.2cm}, clip=true, width=\linewidth]{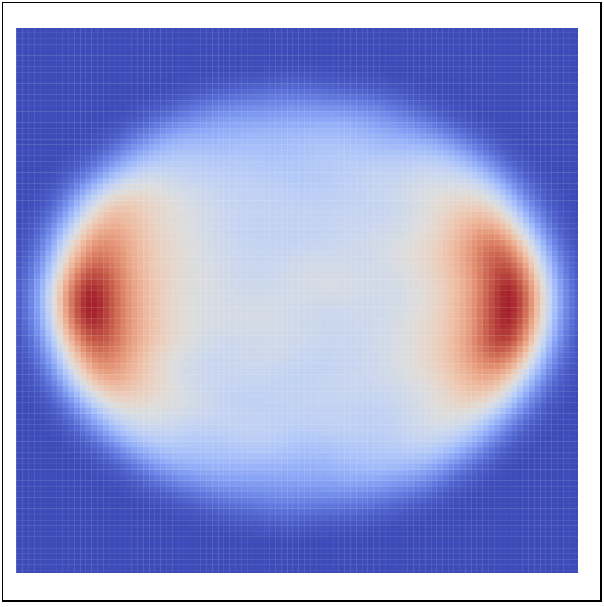}
\end{subfigure}
\vspace{-0.7cm}
\caption{Dimensional Collapse: (left) the $\ell_2$ ball \vs (right) the Poincaré disk (note the density collapse (non-red)).\label{fig:dc}}
\vspace{-0.8cm}
\end{wrapfigure}

Learning features in hyperbolic spaces has drawn a lot of interest~\cite{bronstein2017geometric,sun2021hgcf}. Hyperbolic spaces are characterized by their negative curvature,
where  %(\cf zero or positive). In hyperbolic spaces, 
the distance between two points can grow exponentially %with their separation, 
in contrast to the Euclidean space  with the linear distance growth. Compared with the Euclidean space, hyperbolic spaces have  numerous advantages: 1) Geodesic distance measure~\cite{bronstein2017geometric},
2) Better representation of hierarchical structures~\cite{adcock2013tree},
3) Increased capacity~\cite{bronstein2017geometric}, and 4) Improved generalization~\cite{krioukov2010hyperbolic}. 
Hyperbolic space helps embed hierarchical graph structures \cite{bronstein2017geometric}. However, past
works~\cite{chen2020simple, chen2020improved, gao2021simcse,zhang2022costa,wang2020understanding, zhang2024geometric} have not been able to successfully apply CL for hyperbolic graph embeddings. Despite promising results of CL \cite{chen2020simple,zhang2022costa,gao2021simcse, zhang2023mitigating}, such CL models are limited in their ability to model complex patterns that exhibit hierarchical structures. Specifically, in typical CL, embeddings are arranged on a hypersphere in the Euclidean space and compared by %the inner (dot) product or 
the Cosine distance (\figref{fig:sphere_vs_hyperboola} ({\em left})). The hyperbolic space is a natural candidate for modeling hierarchical structures \cite{bronstein2017geometric, krioukov2010hyperbolic} for self-supervised learning via CL.

\begin{wrapfigure}{r}{0.4\linewidth}
\vspace{-0.4cm}
    \includegraphics[width=0.4\textwidth]{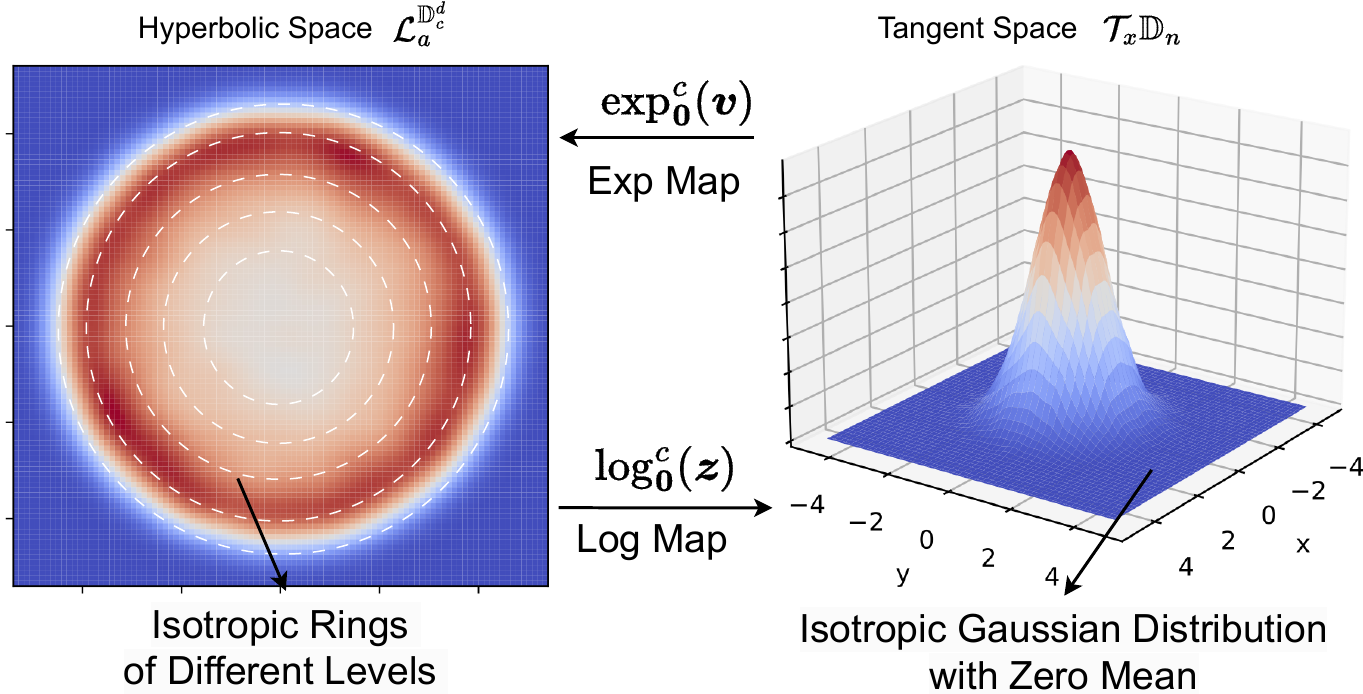}
    \vspace{-0.3cm}
    \caption{Mapping isotropic Normal distribution from the tangent plane at $\mathbf{0}$ to the ambient space of the hyperbolic manifold results in the isotropic ring-shaped density, known as the Wrapped Normal distribution.
    %(and back).
    \label{fig:hyper_tanh}}
    \vspace{-0.3cm}
\end{wrapfigure}

Dimensional collapse (DC)~\cite{liu2019spectral}, shown in \figref{fig:dc}, is a phenomenon occurring in CL %representation learning 
where the learned embeddings occupy only small parts of the feature space. As the  embedding  dimension is
usually high, the DC prevents learning  
   diverse information in  high-dimensional space, limiting its representativeness.

Wang \etal~\cite{wang2020understanding} note that the quality of representation produced via CL is characterized by two key factors: alignment and uniformity~\cite {wang2020understanding}. Good alignment ensures the preservation of distinct information contained by samples, while excessive alignment may lead to dimensional collapse where features become overly concentrated on specific points or subspace~(\figref{fig:dc})~\cite{jing2021understanding,grill2020bootstrap, zhang2023spectral}. To counteract this effect, the so-called uniformity constraints are often imposed to encourage samples to span the entire space evenly to increase the feature diversity. In the Euclidean space, maximizing the pairwise distance between samples on the hypersphere leads to a roughly uniform distribution on the surface. %~\cite{wang2020understanding}.

However, since the hyperbolic space is non-compact and has infinite volume, maximizing the uniformity directly in the hyperbolic space leads to pushing all samples towards the boundary of the Poincaré ball. Thus, uniformity cannot be achieved. Therefore, alignment with and without uniformity results in the Hyperbolic Dimensional Collapse (HDC), which can be characterized by decline of Effective Rank \cite{7098875} of features in the ambient space of the hyperbolic manifold\footnote{The ambient space of the Hyperbolic manifold is the space where the hyperbola lies.}. The HDC is illustrated in Figures \ref{fig:hyper_tree_1}, \ref{fig:hyper_tree_3} \& \ref{fig:hyper_tree_4}.

In this paper, we investigate the phenomenon of HDC. We adopt the Poincaré model and introduce a novel framework called Hyperbolic Graph Contrastive Learning (HyperGCL). The primary goal of HyperGCL is to generate high-quality graph embeddings that avoid the HDC and effectively apply it to various downstream tasks. To this end, we revisit the concept of alignment and uniformity for the Hyperbolic embeddings in Section \ref{sec:prelim}. 
% We adopt the Hyperbolic distance\footnote{Hyperbolic distances are known to approximate tree/hierarchical distances~\cite{NEURIPS2019_0ec04cb3}.} to measure the alignment in the hyperbolic space\rex{the last sentence is unnecessary. you will talk about this detail in method section}. 

\definecolor{beaublue}{rgb}{0.85, 0.9, 0.95}
\definecolor{blackish}{rgb}{0.2, 0.2, 0.2}

%\vspace{-0.2cm}
\begin{tcolorbox}[width=1\linewidth, colframe=blackish, colback=beaublue, boxsep=0mm, arc=2mm, left=2mm, right=2mm, top=1mm, bottom=1mm]
The notion of uniformity is only defined for a hypersphere but not hyperbolic manifolds. % especially for a manifold with an infinite volume of manifold hyper-surface
Thus, instead of talking about the ``uniformity'' of hyperbolic space, we  talk about the level-wise isotropic rings in the ambient space of the hyperbolic manifold (\figref{fig:hyper_tanh} (left)). We are interested in imposing high density of features uniformly distributed along the ring  circumference close to the boundary of Poincaré ball. We call it the outer isotropic shell or simply put, isotropic shell. To optimize it, we discover that enforcing an isotropic Normal distribution on the tangent plane at $\mathbf{0}$ to the Poincaré ball results in such an  isotropic shell, as shown in~\figref{fig:hyper_tanh}. We are interested in such an ambient space because it can embed a tree with large depth and uniformly distributed leaves (Fig. \ref{fig:hyper_tree_2}) according to low distortion delaunay-based tree embedding \cite{NEURIPS2019_0ec04cb3}.
\end{tcolorbox}

\vspace{0.2cm}
Below we summarize our contributions as follows:

\vspace{-0.1cm}
\renewcommand{\labelenumi}{\roman{enumi}.}
\begin{enumerate}[leftmargin=0.6cm]
\item We investigate the problem of dimension collapse  in  the hyperbolic contrastive learning, and we associate the dimensional collapse with the tree ``leaf collapse'' and ``height collapse'', describing it from the point of view  that  trees can be embedded in the hyperbolic plane with low distortion \cite{sarkar2011low,ganea2018hyperbolic,NEURIPS2019_0ec04cb3}.
\item We propose a novel graph contrastive learning framework in the hyperbolic space to generate high-quality graph embedding for the various downstream tasks.
\item To alleviate the dimensional collapse, identified/measured as a drop in Efficient Rank (ERank) of features in the ambient space of the hyperbolic manifold, we propose a new isotropic Gaussian loss operating on the tangent space of manifold at $\mathbf{0}$ which forces the desired feature distribution (isotropic shell) in the ambient space of the hyperbolic manifold. %by mapping it from the tangent plane at $\mathbf{0}$ to the Hyperbolic space.
We show that imposing the isotropic Gaussian loss on the tangent space increases the ERank of feature representations on the tangent plane, and that the ERank measured in the ambient space correlates with the ERank measured on the tangent plane. Maintaining high ERank stops backbone succumbing to DC.
\end{enumerate}

\section{Preliminaries}
\label{sec:prelim}
\noindent \textbf{Notation.} A graph $G$ with node features is represented as $G=(\mathbf{X}, \mathbf{A})$, where $\mathbf{X} \in \mathbb{R}^{N \times d_{x}}$ is the feature matrix with $\boldsymbol{x}_i$ as the feature vector of node $v_i$, and $\mathbf{A} \in \{0,1\}^{n \times n}$ is the adjacency matrix where $A_{ij}=1$ indicates an edge between nodes $i$ and $j$. The degree $d_i$ denotes the number of edges incident to node $i$, and $\mathbf{D}$ is the diagonal degree matrix with $D_{ii}=d_i$. For notation, $\|\boldsymbol{x}\|_{2}$ is the Euclidean norm of vector $\boldsymbol{x}$, $x_i$ is the $i$-th entry of vector $\boldsymbol{x}$, $x{ij}$ is the $(i,j)$-th entry of matrix $\mathbf{X}$, $\operatorname{diag}(\boldsymbol{x})$ creates a diagonal matrix from vector $\boldsymbol{x}$, and $\operatorname{tr}(\mathbf{X})$ is the trace of matrix $\mathbf{X}$.

\vspace{0.1cm}
\noindent\textbf{Alignment \& Uniformity on Hypersphere $\mathcal{S}^{d-1}$. }% for Contrastive Learning}
% \noindent\textbf{Alignment and Uniformity.} 
Contrastive learning is commonly used to encourage learned feature representations for positive pairs to be similar while separating features from randomly sampled negative pairs. It is known that representations should capture the information shared between positive pairs while remaining invariant to  nuisance/noise factors~\cite{tschannen2019mutual, wu2018unsupervised}. If the representations reside on the hypersphere
$\mathcal{S}^{d-1}$ ($l_2$  ball), Wang \etal \cite{wang2020understanding} argue that the above properties are achieved by optimizing the following objective over positive/negative pairs sampled from the so-called positive  and data distributions,   $p_{\text{pos}}$ and $p_{\text{data}}$:
\begin{equation}
\begin{aligned}
\mathcal{L}_{cont} = \!\!\!\!\underbrace{\underset{(\boldsymbol{x}, \boldsymbol{y}) \sim p_{\text {pos }}}{\!\!\!\!\mathbb{E}}\!\!\!\!\!\|f(\boldsymbol{x})-f(\boldsymbol{y})\|_2^2}_{\mathcal{L}^{\mathbb{R}^d}_{A}} +\underbrace{\log \!\!\!\!\!\underset{\substack{\text { i.i.d. } \\
\boldsymbol{x}, \boldsymbol{y}  \sim p_{\text {data }}}}{\mathbb{E}}\!\!\!\!\!\!\left[e^{-t\|f(\boldsymbol{x})-f(\boldsymbol{y})\|_2^2}\right]}_{\mathcal{L}^{\mathbb{R}^d}_{U}},
\end{aligned}
\label{eq:euc_align_uni}
\vspace{-5px}
\end{equation}
where $f(\cdot)$ denotes the encoder with the $l_2$-normalized output, \ie, $\|f(x)\|_2 = 1$. One can observe that
%\begin{itemize}[leftmargin=0.6cm]
    %\item 
    (i) \textbf{Alignment} ${\mathcal{L}^{\mathbb{R}^d}_{A}}$ makes two samples of a positive pair to be mapped to nearby feature vectors, and thus be (mostly) invariant to undesired nuisance/noise factors, %and the \item 
    whereas (ii) \textbf{Uniformity} ${\mathcal{L}^{\mathbb{R}^d}_{U}}$  forces feature vectors be roughly uniformly distributed on the unit ball $\mathcal{S}^{d-1}$, diversifying features. %preserving as much information of the data as possible.
%\end{itemize}

\vspace{0.1cm}\noindent\textbf{Dimensional Collapse (DC).} Often referred to as spectral collapse~\cite{liu2019spectral}, DC is common in representation learning as shown in~\figref{fig:dc}. It occurs when the embedding space is dominated by a small number of large singular values, while the remaining singular values decay rapidly as the training progresses. This phenomenon %poses a challenge for representation learning, as it 
limits the representational power of high-dimensional spaces by restricting the diversity of information that can be learned. In the framework of optimizing the alignment and uniformity, optimizing the uniformity helps alleviate the DC by encouraging features to be uniformly distributed in the entire latent space. % and span the entire space rather than some lower-dimensional subspace.

\begin{wrapfigure}{r}{0.4\linewidth}
\vspace{-1.7cm}
    \includegraphics[width=0.4\textwidth]{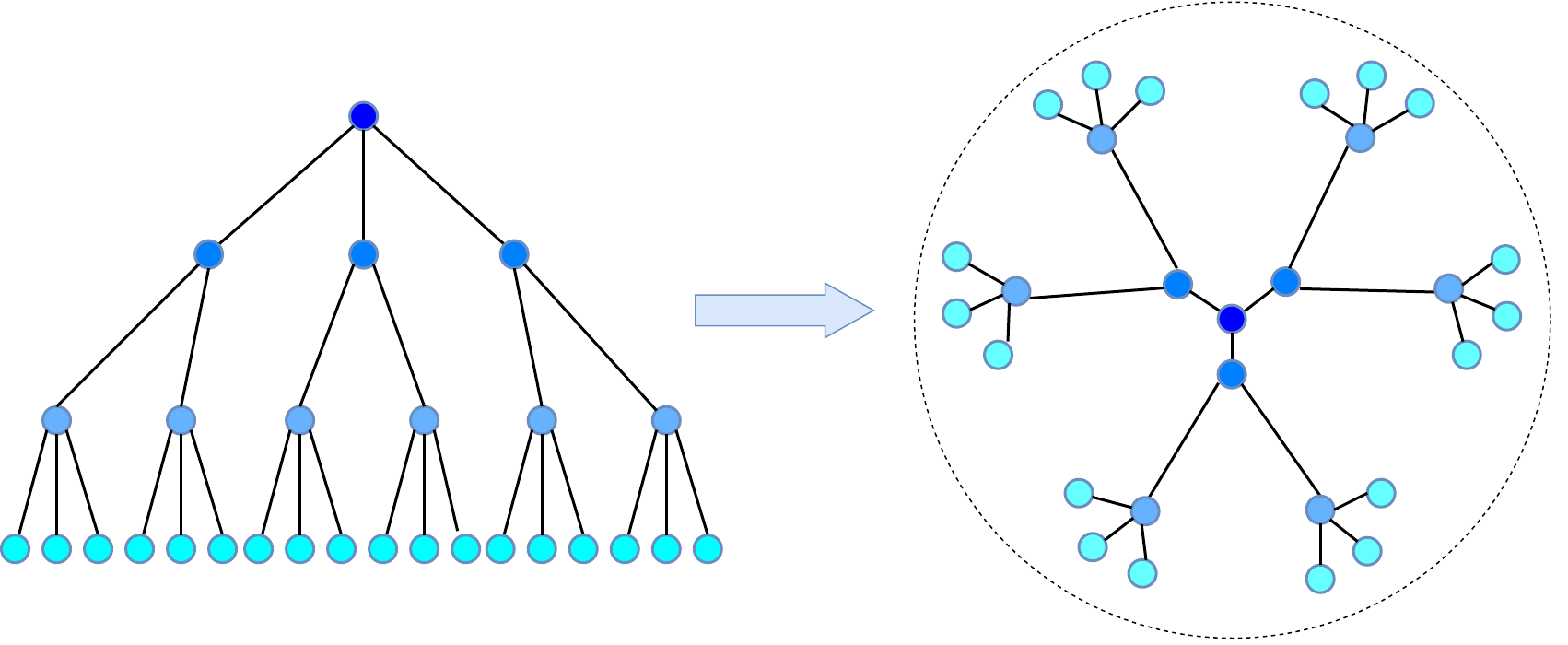}
    \vspace{-0.3cm}
    \caption{
   An example of a low distortion embedding ({\em right}) of a tree ({\em left}) in the hyperbolic plane \cite{sarkar2011low} where distances between direct neighbor nodes are preserved exactly, while non-neighbor distances enjoy at most $1\!+\!\varepsilon$ distortion factor. 
    %An example of a good embedding in hyperbolic space that preserves the local dependencies while maintaining a hierarchical relationships of the entities.
    }
    \label{fig:hypebolic_tree}
 \vspace{-0.8cm}
\end{wrapfigure}

\subsection{Hyperbolic Geometry}
To describe our HyperGCL, we first briefly review the hyperbolic geometry and its  properties that are leveraged by our model.

\vspace{0.1cm}
\noindent\textbf{Poincaré Ball Model.} A hyperbolic space $\mathbb{H}$ is a complete, connected Riemannian manifold with constant negative sectional curvature. Cannon \etal \cite{cannon1997hyperbolic} describe five common hyperbolic models. In this paper, we choose the Poincaré ball $\mathbb{D}_c^d:=\big\{\boldsymbol{p} \in \mathbb{R}^d \mid\|\boldsymbol{p}\|^2<\frac{1}{c}\big\}$ as our basic model~\cite{nickel2017poincare, tifrea2018poincar}, where $\frac{1}{c}>0$ is the radius of the ball. The Poincaré ball is coupled with a Riemannian metric $g_{\mathbb{D}}(\boldsymbol{p})=\frac{4c}{\left(1-\|\boldsymbol{p}\|^2\right)^2}\, g_{\mathbb{E}}$, where $\boldsymbol{p}\in \mathbb{D}_c^d$, whereas $g_{\mathbb{E}}$ is the canonical metric of the Euclidean space. %In addition, the Poincaré ball can be viewed as a natural counterpart of the hypersphere as it allows all directions, unlike the other models such as the halfspace or hemisphere models that have constraints on the directions. 
The hyperbolic space is globally differomorphic to the Euclidean space.

\begin{definition}[Riemannian distance in $\mathbb{D}_c^d$]
     For $\boldsymbol{p}, \boldsymbol{q} \in \mathbb{D}_c^d$, the Riemannian distance on the Poincaré ball induced by its metric $g_{\mathbb{D}}$ is defined as $D_{c}(\boldsymbol{p}, \boldsymbol{q})=\frac{2}{\sqrt{c}} \tanh ^{-1}\left(\sqrt{c}{\|-p \oplus q\|_2}\right)$ where $\oplus$ is the Möbius addition and it is clearly differentiable.
\end{definition}
\vspace{-0.1cm}
\begin{definition}[Tangent Space]
 The tangent space $\mathcal{T}_{\boldsymbol{x}} \mathbb{D}^d_c\left(\boldsymbol{x} \in \mathbb{D}_c^d\right)$ is defined as the first-order approximation of $\mathbb{D}^d_c$ around ponit $\boldsymbol{x}$ : $\mathcal{T}_{\boldsymbol{x}} \mathbb{D}^d_c:=\big\{\boldsymbol{v} \in \mathbb{R}^{d+1}:\langle\boldsymbol{v}, \boldsymbol{x}\rangle=0\big\}$.
\end{definition}

To perform operations in the hyperbolic space, the bijective map from $\mathbb{R}^d$ to $\mathbb{D}_c^d$ maps Euclidean vectors to the hyperbolic space. %and vice versa is define 
The so-called exponential map performs such a mapping (the logarithmic map performs the inverse mapping). 
\vspace{-0.1cm}
\begin{definition}[Exponential/Logarithmic Map]
The exponential map $\exp _{\boldsymbol{x}}^c(\cdot)$ is a function from $T_{\boldsymbol{x}} \mathbb{D}_c^d \cong \mathbb{R}^d$ to $\mathbb{D}_c^d$. The logarithmic map $\log _{\boldsymbol{x}}^c(\cdot)$ maps from $\mathbb{D}_c^d$ to $T_{\boldsymbol{x}} \mathbb{D}_c^d$. These maps are defined as:
\begin{equation}
\begin{aligned}
        \exp _{\boldsymbol{x}}^c(\boldsymbol{v})\!&=\!\boldsymbol{x} \oplus\Big(\tanh \Big(\frac{\sqrt{c}\lambda_{\boldsymbol{x}}^c}{2}\|\boldsymbol{v}\|\Big) \frac{\boldsymbol{v}}{\sqrt{c}\|\boldsymbol{v}\|}\Big) \quad\text{ and } \\
    \log _{\boldsymbol{x}}^c(\boldsymbol{y})\!&=\!\frac{2}{\sqrt{c} \lambda_{\boldsymbol{x}}^c} \tanh^{-1}\big(\sqrt{c}\left\|-\boldsymbol{x} \oplus \boldsymbol{y}\right\|\big) \frac{-\boldsymbol{x} \oplus \boldsymbol{y}}{\left\|-\boldsymbol{x} \oplus \boldsymbol{y}\right\|},
\end{aligned}
\end{equation}
where $\lambda^c_{\boldsymbol{x}}=\frac{2}{1-c\|\boldsymbol{x}\|^2}$ is the conformal factor that scales the local distances and $\|\cdot\|$ is the $\ell_2$ norm.
\end{definition}
%In practice, 
We use  $\exp _{\boldsymbol{0}}^c(\cdot)$ and $\log_{\boldsymbol{x}}^c(\cdot)$ to transition between the Euclidean and Poincaré ball representations. % of a vector.

\vspace{0.1cm}
\noindent
\textbf{Learning Hierarchical Representations with Poincaré Embeddings.}
%\yifei{
Hyperbolic spaces have emerged as a compelling alternative for capturing hierarchical structures within textual and graph data \cite{paeng2011brownian,nickel2017poincare,NEURIPS2019_0ec04cb3}. Hyperbolic spaces offer a continuous framework akin to trees, as depicted in \figref{fig:hypebolic_tree} where trees can be seamlessly embedded with a minimal distortion (factor $1\!+\!\varepsilon$) into the Poincaré disc model~\cite{sarkar2011low}. The hyperbolic manifold surface %Poincaré disc's surface 
area exhibits an exponential growth relative to its radius (viewed from the surface of the Poincaré disc), mirroring the proliferation of leaves as the depth of tree increases. Moreover, the operations on  hyperbolic spaces are differentiable and thus applicable to deep learning.

\section{Methodology}
\subsection{Motivation}
% \subsection{Dimensional Collapse in Poincaré Ball}

%
%\begin{figure}[th]
\begin{wrapfigure}{r}{0.4\linewidth}
\vspace{-2.7cm}
%\begin{subfigure}[t]{0.49\linewidth}
    \centering
\begin{subfigure}{0.48\linewidth}
        \centering
        \includegraphics[width=1\linewidth]{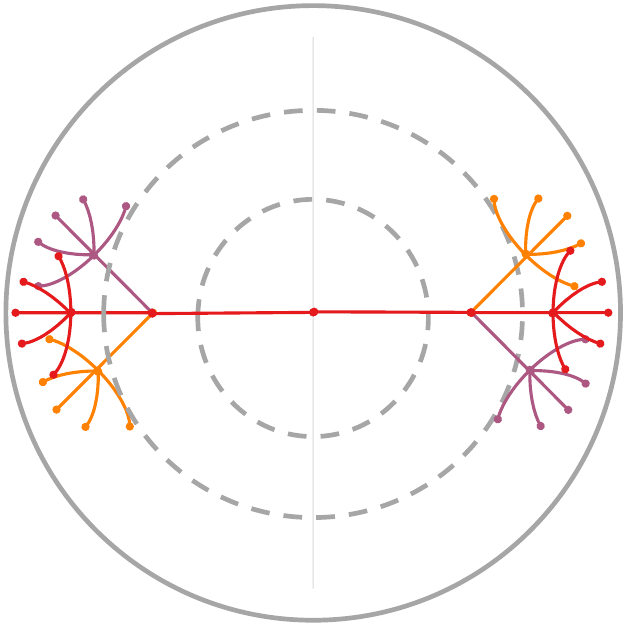}
        \caption{Leaf collapse}
        \label{fig:hyper_tree_1}
        \vspace{0.2cm}
\end{subfigure}
\hspace{1pt}
\begin{subfigure}{0.48\linewidth}
        \centering
    \includegraphics[width=1\linewidth]{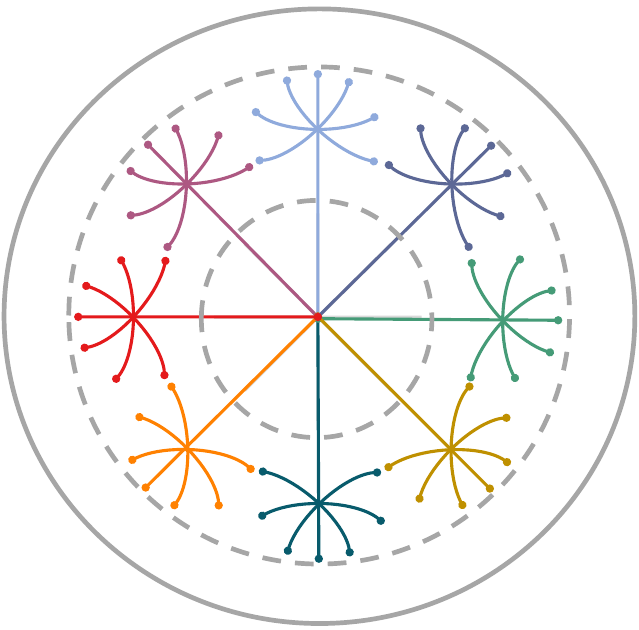}
        \caption{Height collapse}
        \label{fig:hyper_tree_3}
        \vspace{0.2cm}
\end{subfigure}
%\hspace{5pt}
\begin{subfigure}{0.48\linewidth}
        \centering
    \includegraphics[width=1\linewidth]{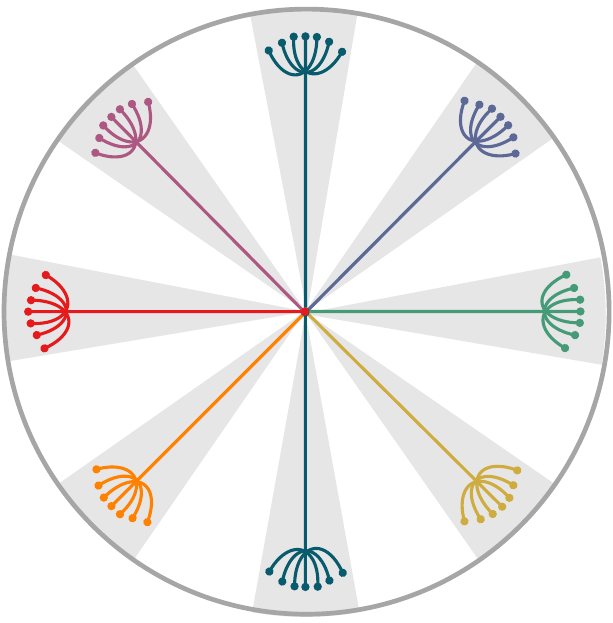}
        \caption{Leaf+height\\ collapse}
        \label{fig:hyper_tree_4}
\end{subfigure}
\hspace{1pt}
\begin{subfigure}{0.48\linewidth}
        \centering
    \includegraphics[width=1\linewidth]{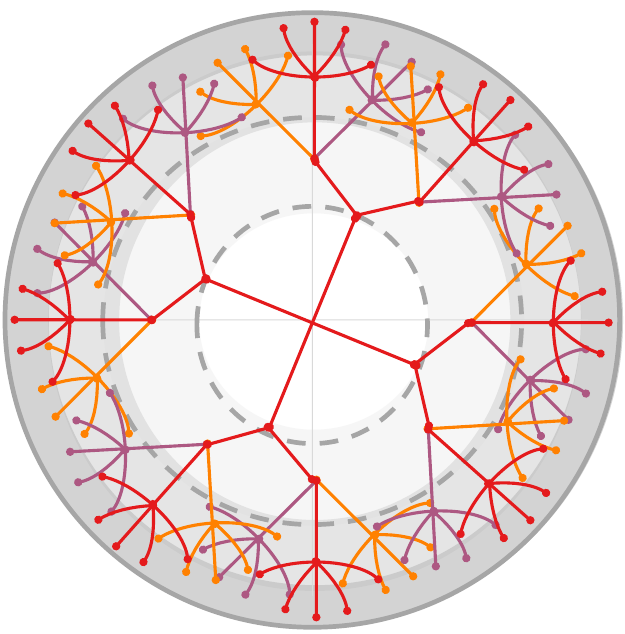}
        \caption{Healthy tree\\ embedding}
        \label{fig:hyper_tree_2}
\end{subfigure}
%\vspace{-0.2cm}
    \caption{Several trees embedded in  Poincaré disk. Fig. \ref{fig:hyper_tree_1}, \ref{fig:hyper_tree_3} \& \ref{fig:hyper_tree_4} show tree embeddings under the dimensional collapse, whereas Fig. \ref{fig:hyper_tree_2} is a plausible ``healthy'' tree embedding (see text for details). Gray rings indicate the density of the underlying distribution.\label{fig:tree_emb}}
    \vspace{-0.6cm}
%\end{figure}
\end{wrapfigure}

Preventing the dimensional collapse in representation learning is of utmost importance. Wang and Isola \cite{wang2020understanding} provide a new perspective on this issue by introducing a uniformity metric that limits the degree of collapse in representations. They have demonstrated that in order  to maintain diverse information for downstream tasks,  the 
learned embeddings must be evenly distributed on the hypersphere $\mathcal{S}^{d-1}$. In the Euclidean space, a typical collapse mode occurs when features collapse to a single point or a subspace, as  in \figref{fig:dc} (left).

\yifei{However, the collapse mode (\figref{fig:dc} (right)) for the hyperbolic model differs due to the exponential growth of hyperbolic space. Direct application of uniformity fails to ``fill'' that infinite volume uniformly. The uniformity in the ambient space of the hyperbolic manifold is also meaningless for curved manifold. In other words, the uniformity loss for contrastive learning makes no sense in hyperbolic spaces which motivate us to rethinking {\em what  the dimensional collapse  in the hyperbolic space is and how to prevent it.}

To this end, we identify two types of collapse caused by poor utilization of hyperbolic feature space (Section~\ref{sec:hdc}). To mitigate the Hyperbolic Dimensional Collapse (HDC), we impose isotropy shell densities in the ambient space, arguing this is a meaningful regularization in hyperbolic spaces in place of traditional uniformity. Our theoretical analysis reveals that imposing the radially symmetric distributions (\eg, the normal distribution) in tangent space results in angularly uniform isotropic shells with density proportional to the distance to the origin. We show that optimizing our isotropy loss leads to the same maximum as optimizing the effective rank (measure of DC).

}
\subsection{Hyperbolic Dimensional Collapse}~\label{sec:hdc}
% \vspace{0.1cm}
% \noindent\textbf{Hyperbolic Dimensional Collapse}
% \begin{tcolorbox}[width=1\linewidth, colframe=blackish, colback=beaublue, boxsep=0mm, arc=2mm, left=2mm, right=2mm, top=5mm, bottom=2mm]
The dimensional collapse in the hyperbolic space results in an imbalanced distribution in the ambient space of the hyperbolic manifold. \figref{fig:tree_emb} shows the collapse mode may be associated with inadequately constructed tree embeddings. A data point near the center of the Poincaré ball is considered the root node, while data points near the boundary of the ball are leaf nodes. One possible collapse mode is referred by us to as ``\textbf{leaf collapse},'' where the embedding of the tree node collapses into several dense regions at a specific level of the tree. \figref{fig:hyper_tree_1} shows an example of this collapse, which is similar to the collapse observed on the hypersphere when the model maps points only in the few parts of the ball. Another collapse mode, called by us as ``\textbf{height collapse}'' (\figref{fig:hyper_tree_3}), occurs when the underlying hierarchy results in a shallow tree, limiting discriminative hierarchical relationships. \figref{fig:hyper_tree_4} includes both ``leaf collapse'' and ``height collapse'' that limit the efficient use of the embedding space, and the expressive power of the hyperbolic space for downstream tasks. Figure \ref{fig:hyper_tree_2} shows a healthy embedding. The gray rings illustrate density of feature vectors in the ambient space of hyperbolic manifold. Notice the outer shell (ring) with isotropic density along circumference which has the most density, indicating we attained a full tree depth due to the highest density close to the boundary of the  Poincaré ball. In contrast, the density (gray regions) in Fig. \ref{fig:hyper_tree_4} show the dimensional collapse as encoder ends up producing partially empty ambient space with a reduced Effective Rank of ambient features. The actual feature collapse is caused by the encoder: that is why we monitor the ambient space. 
% \end{tcolorbox}

\subsection{Hyperbolic Graph CL (HyperGCL)}

% Answer: [trim={left bottom right top},clip]
\begin{figure}[t]
\vspace{-0.3cm}
    \centering
    \vspace{-0.3cm}
    \includegraphics[trim={0 0 3.0cm 0},clip=true, width=0.95\linewidth]{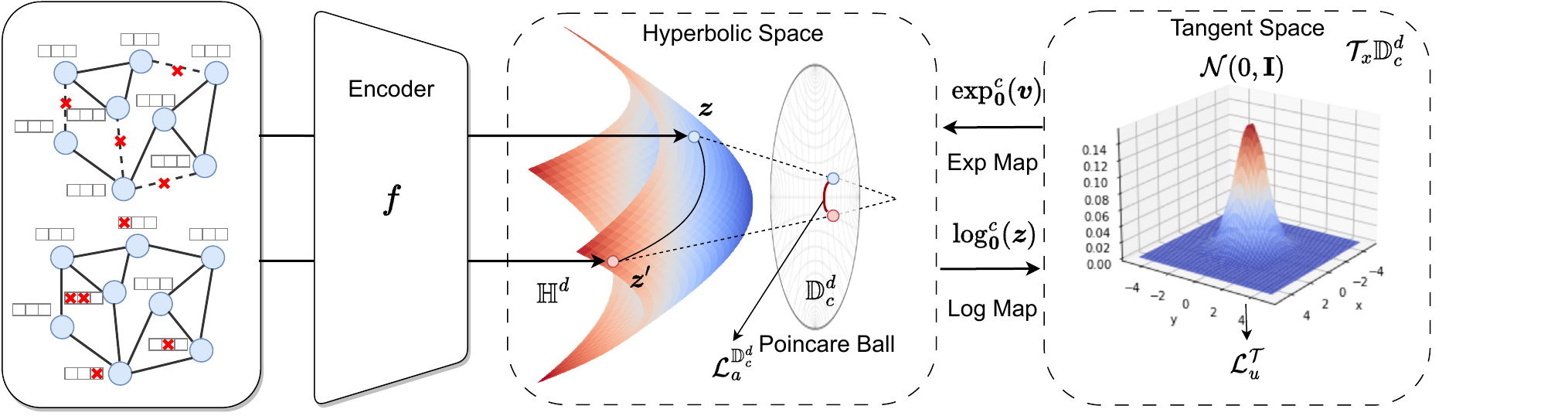}
    %\vspace{-0.3cm}
    \caption{Our Hyperbolic Graph CL framework (HyperGCL).}
    \label{fig:overall}
 \vspace{-0.5cm}
\end{figure}

\noindent\textbf{Overall Framework (\figref{fig:overall}).} We propose %a framework called 
Hyperbolic Graph Contrastive Learning (HyperGCL) that leverages the theoretical benefits of the hyperbolic space to impose an isotropic shell level-wise in the ambient space and the alignment directly in the hyperbolic space. %in $\mathbb{D}_c^d$. 
%The overall pipeline of our framework is illustrated in~\figref{fig:overall}.
% Appendix \ref{sec:notation} gives our notations.

Following the typical graph contrastive learning~\cite{zhang2022costa,zhu2020deep}, we first generate two augmented  graph views $G = (\mathbf{X}, \mathbf{A})$ and $G^\prime = (\mathbf{X}^{\prime}, \mathbf{A}^{\prime})$ by randomly dropping edges and node attributes. Then we adopt Graph Neural Network (GNN), \ie, the GCN encoder $f(\cdot)$, to transform the the augmented graph views into the latent node embedding $\mathbf{Z}$ and $\mathbf{Z^\prime}$. 
% We project $\mathbf{Z}$ and $\mathbf{Z^\prime}$ (the features of GNN $f(\cdot)$) into $\mathbb{D}_c^d$ as in Appendix \ref{sec:proj1}.
%\vspace{-0.2cm}
\begin{tcolorbox}[width=1\linewidth, colframe=blackish, colback=beaublue, boxsep=0mm, arc=2mm, left=2mm, right=2mm, top=1mm, bottom=1mm]
\textbf{Implicit hyperbolic space.}
Assume the output space of the graph neural network $f_\Theta(\cdot)$ is in the Poincaré ball $\mathbb{D}_c^d$, we project the all the node embedding to the $\mathbb{D}_c^d$ as
\begin{equation}
\boldsymbol{z}:= \begin{cases}\boldsymbol{z} & \text { if }\|\boldsymbol{z}\|_2 \leq (1\!-\!\epsilon)\frac{1}{\sqrt{c}} \\ (1-\epsilon) \frac{\boldsymbol{z}}{\sqrt{c}\|\boldsymbol{z}\|_2} & \text { else }\end{cases}.
\label{eq:poj}
\end{equation}
%
% where $\epsilon\!\geq\!0$ is a small constant preventing features reaching the disk boundary.

We set a small margin $\epsilon>0$ to prevent infinite volume of the hyperbolic manifold by keeping features away from the disk boundary. Thus, we prevent infinite depth trees and ensure numerical stability.  %is on the Poincaré ball $\mathbb{D}_c^d$,
 Our latent representation space is implicitly made hyperbolic as neural networks are universal approximators so $\exp(\cdot)$ mapping from the Euclidean to Hyperbolic space is implicitly learned if the Hyperbolic loss is applied to impose hyperbolic organization of feature space. This is a known  alternative \cite{fang2023poincare} to encoders with exponential maps, and hyperbolic encoders. % (\eg,  HGCN~\cite{chami2019hyperbolic}.)  
 A  study of these two setting is in Appendix~\ref{sec:gnn_hypergnn}.
%Note that the latent representation space is implicitly made hyperbolic as neural networks are universal approximators so $\exp(\cdot)$ mapping from the Euclidean to Hyperbolic space is implicitly learned if the Hyperbolic loss is applied to impose hyperbolic organization of feature space. Such design is known \cite{fang2023poincare,fang2021kernel} and it outperforms fully hyperbolic GNN encoder, \eg, HGCN~\cite{chami2019hyperbolic}. A  study of these two setting is in Appendix~\ref{sec:gnn_hypergnn}.
\end{tcolorbox}
% \vspace{-0.2cm}

Denote $(\mathbf{Z},\mathbf{Z^\prime})$ as $N$ pairs of node embeddings $\big\{(\boldsymbol{z}_i,\boldsymbol{z}^{\prime}_i)\big\}_{i=1}^N$. The goal of Hyperbolic Graph contrastive learning (HyperGCL) is to find a hyperbolic organization of representation that minimizes the distance between different augmented embeddings $\boldsymbol{z}_i$ and $\boldsymbol{z}^{\prime}_i$ of the same sample node $i$ by maximizing hyperbolic alignment $\mathcal{L}^{\mathbb{D}_c^d}_{A}$, and increase the uniform density along a shell of distinct node embeddings $\boldsymbol{z}_i$ and $\boldsymbol{z}_j$ by promoting the outer shell isotropy $\mathcal{L}^{\mathbb{D}_c^d}_{U}$. We define our new alignment and the outer shell isotropy loss terms % in $\mathbb{D}_c^d$ 
and propose to optimize the following objective: 
 \begin{equation}
     \mathcal{L}_{HyperGCL}^{\mathbb{D}_c^d}(\mathbf{Z}, \mathbf{Z}^\prime)= \mathcal{L}^{\mathbb{D}_c^d}_{A}(\mathbf{Z}, \mathbf{Z}^\prime) + \lambda\,\mathcal{L}^{\mathbb{D}_c^d}_{U}(\mathbf{Z}, \mathbf{Z}^\prime).
     \label{eq:align_uni_loss}
 \end{equation}

In what follows, our focus is on how to design the alignment and especially  the outer shell isotropy for the contrastive hyperbolic learning to ensure the similar hierarchically organized datapoints enjoy small hyperbolic distance, and prevent the dimensional collapse of the ambient space of the hyperbolic manifold. 

% In the following, our primary objective is to address the design of an alignment metric capable of effectively capturing hierarchical data-invariant information. Additionally, we aim to develop a uniformity metric that exhibits high sensitivity in detecting dimensional collapse within the hyperbolic space.

%\setcounter{equation}{6}

\subsection{Applying Alignment and Outer Shell Isotropy for Hyperbolic Learning}

\vspace{0.1cm}
\noindent\textbf{Optimizing Alignment in $\mathbb{D}_c^d$.} 
Hyperbolic spaces are not vector spaces in a traditional sense so the Euclidean operations are not applicable in such spaces. %such as summation and multiplication do not apply. 
Instead, the formalism of M\"{o}bius gyrovector spaces  helps generalize standard operations to hyperbolic spaces. We simply define\footnote{Kindly note we do not claim the use of the hyperbolic distance as a contribution.} the alignment of $\boldsymbol{z}_i$ and $\boldsymbol{z}^{\prime}_i$ as: 
% \begin{equation}
%  f(\boldsymbol{x}) \oplus  f(\boldsymbol{y})=\frac{\left(1+2 c\left\langle f(\boldsymbol{x}),  f(\boldsymbol{y})\right\rangle+c\left\| f(\boldsymbol{y})\right\|^2\right)  f(\boldsymbol{x})+\left(1-c\left\| f(\boldsymbol{x})\right\|^2\right)  f(\boldsymbol{y})}{1+2 c\left\langle f(\boldsymbol{x}),  f(\boldsymbol{y})\right\rangle+c^2\left\| f(\boldsymbol{x})\right\|^2\left\| f(\boldsymbol{y})\right\|^2}
% \end{equation}

%\setcounter{equation}{4}

\vspace{-0.3cm}
\begin{equation}
\mathcal{L}^{\mathbb{D}_c^d}_{A}(\mathbf{Z}, \mathbf{Z}^\prime) = 
\frac{1}{N}\sum_{i=1}^{N} D_c\left( \boldsymbol{z}_i,  \boldsymbol{z}^{\prime}_i\right)=\frac{2}{N\sqrt{c}}\sum_{i=1}^{N} \tanh ^{-1}\left(\sqrt{c}\left\|- \boldsymbol{z}_i \oplus  \boldsymbol{z}_i^{\prime}\right\|\right).
\end{equation}

%\vspace{0.1cm}
\noindent\textbf{Optimizing the Outer Shell Isotropy.} Ways to reduce the dimensional collapse in the hyperbolic space are not obvious. One may naively replace the Euclidean distance in~\eqref{eq:euc_align_uni} by the hyperbolic dist.  $D_c( \boldsymbol{z}_i,  \boldsymbol{z}^{\prime}_i)$:

\vspace{-0.5cm}
\begin{equation}
    {\text{$\mathcal{L}$}}^{\mathbb{D}_c^d}_{U} = {\log \!\!\!\!\!\underset{\substack{\text { i.i.d. } \\
\boldsymbol{z}, \boldsymbol{z}^{-} \sim p_{Z}}}{\mathbb{E}}\!\!\!\!\left[e^{-t D_c\left( \boldsymbol{z},  \boldsymbol{z}^{-}\right)}\right].}
\label{eq:hyper_uni}
\end{equation}

\begin{figure*}[t]
\vspace{-0.3cm}
    \centering
         \begin{subfigure}[c]{0.31\textwidth}
    \centering
         \includegraphics[width=\textwidth]{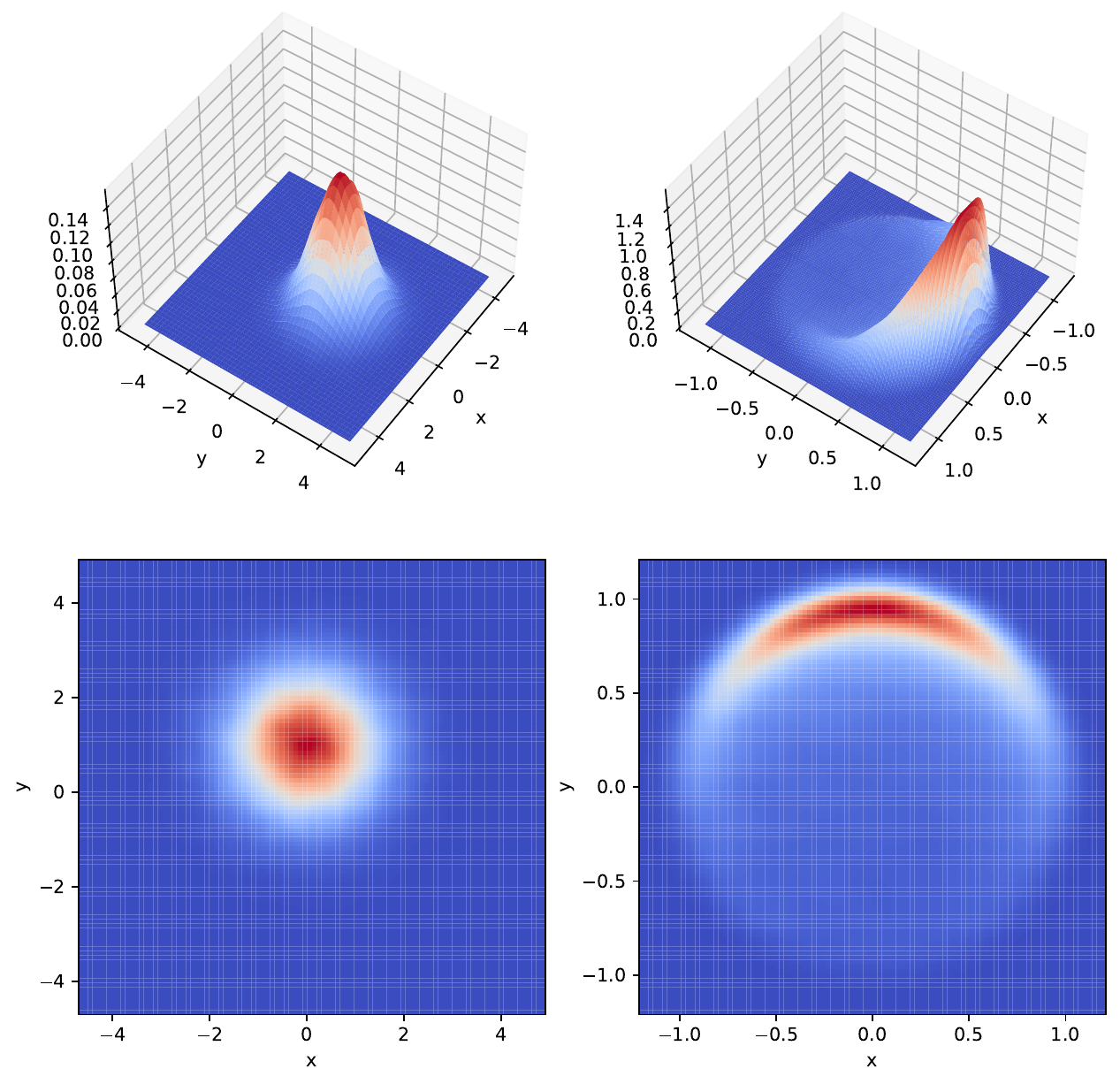}
         \vspace{-0.6cm}
         \caption{$\mathcal{N}([0, 1], \operatorname{diag}([1,1]))$}
         \label{fig:gaussain_dis_a}
    \end{subfigure}
    \hfill
    % \hspace{10pt}
    \begin{subfigure}[c]{0.31\textwidth}
    \centering
         \includegraphics[width=\textwidth]{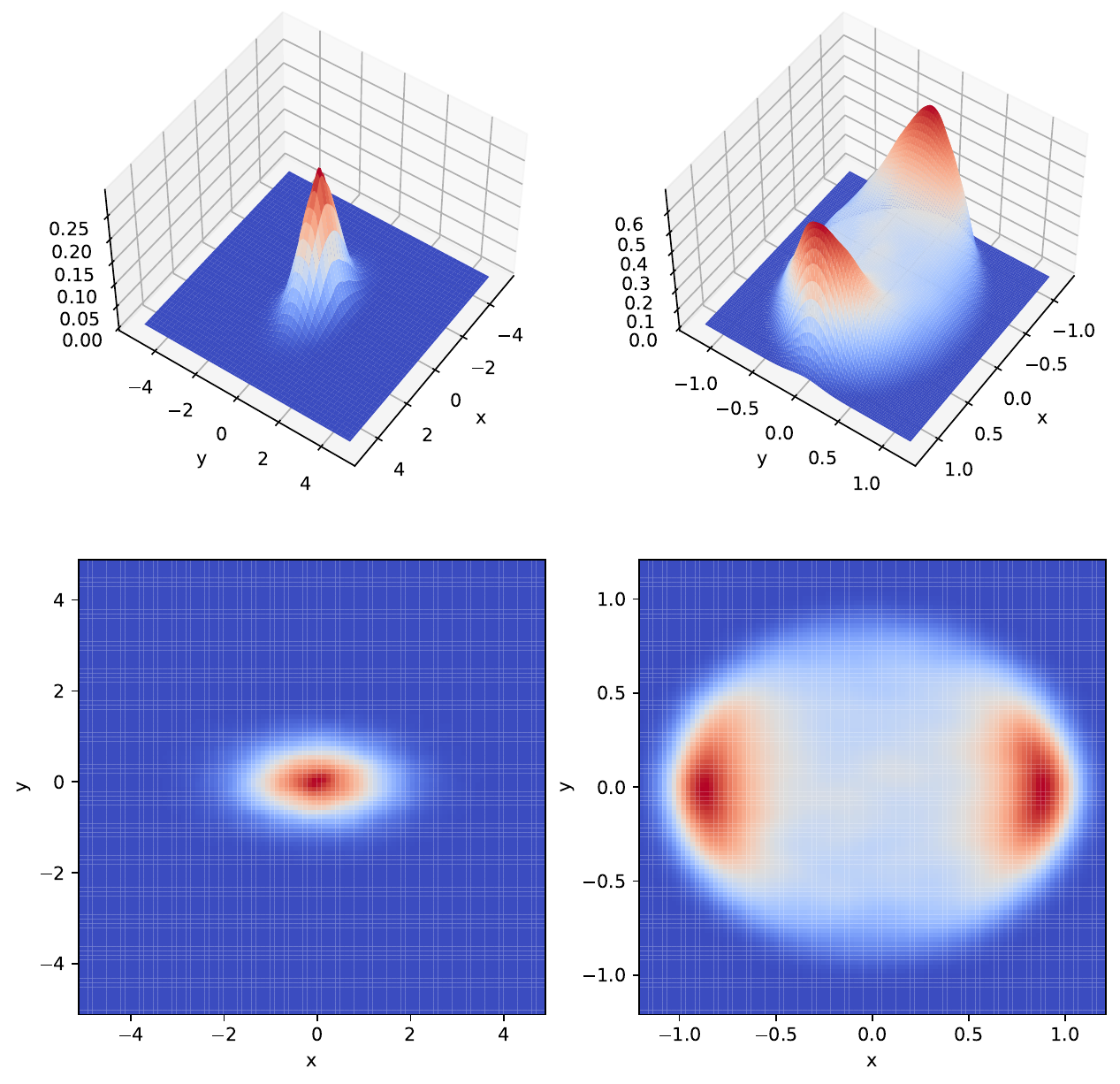}
         \vspace{-0.6cm}
         \caption{$\mathcal{N}([0, 0], \operatorname{diag}([1,0.3]))$}
         \label{fig:gaussain_dis_b}
    \end{subfigure}
    \hfill
     \begin{subfigure}[c]{0.31\textwidth}
    \centering
         \includegraphics[width=\textwidth]{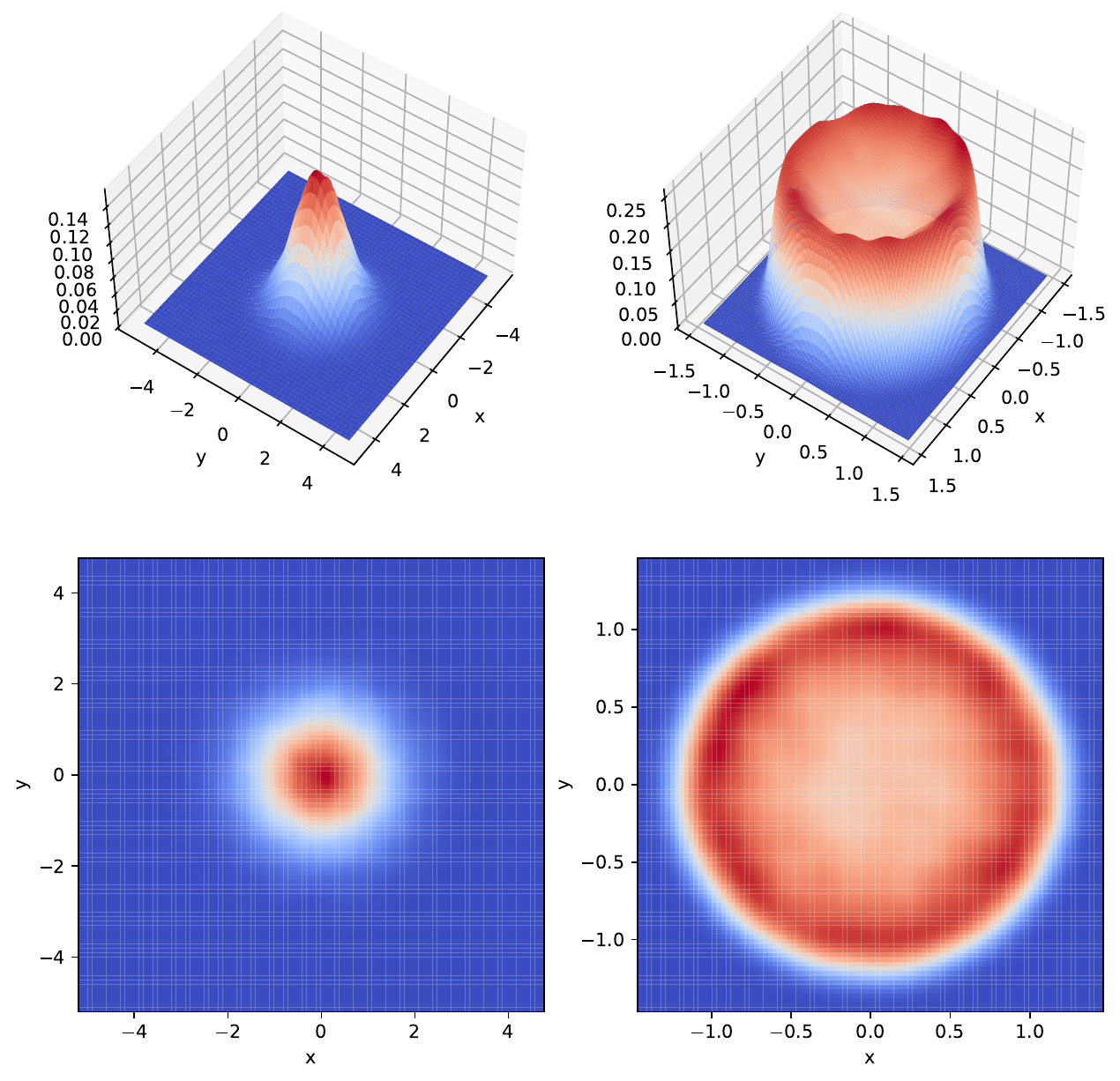}
         \vspace{-0.6cm}
         \caption{$\mathcal{N}([0, 0], \operatorname{diag}([1,1]))$ }
         \label{fig:gaussain_dis_d}
    \end{subfigure}
    \vspace{-0.3cm}
    \caption{Visualization: leaf-level uniformity \wrt non-isotropic/non-zero mean Normal distribution.\label{fig:gaussain_dis}}
    %\vspace{-3px}
\end{figure*}

\begin{table*}[ht]
%\vspace{-0.3cm}
    \centering
    \resizebox{1\linewidth}{!}{
    \begin{tabular}{c c c c c}
    \toprule[2pt] 
       PDF &  $p(\mathbf{y})$ &  $q(\mathbf{y})$ &$D(\mathbf{\Sigma}, \boldsymbol{0})=D_{KL}(p,q)$ & $D(\mathbf{\Sigma}, \boldsymbol{0})+D(\mathbf{\Sigma}^{\prime}, \boldsymbol{0}^{\prime})$ \\
    \midrule
       Normal
       &$\frac{1}{(2 \pi)^{d / 2} \operatorname{det}(\boldsymbol{\Sigma})^{1 / 2}} \exp \left(-\frac{1}{2}(\boldsymbol{y}\!-\!\boldsymbol{\mu})^T \boldsymbol{\Sigma}^{-1}(\boldsymbol{y}\!-\!\boldsymbol{\mu})\right)\;\;$\refstepcounter{equation}(\theequation)
       \label{eq:normal}
       &$\frac{1}{(2 \pi)^{d / 2} } \exp \left(-\frac{1}{2}\lVert\boldsymbol{y}\rVert_2^2\right)\qquad\qquad\!\!\!$\refstepcounter{equation}(\theequation)\label{eq:normal2}
       & $-\frac{1}{2}\log \operatorname{det}(\boldsymbol{\Sigma})\!+\!\frac{1}{2}\left[\operatorname{tr}(\boldsymbol{\Sigma})\!-\!d\right]\;$\refstepcounter{equation}(\theequation)\label{eq:normal3}
       & $-\frac{1}{2}\log \operatorname{det}(\boldsymbol{\Sigma}\boldsymbol{\Sigma}')\!+\!\frac{1}{2}\left[\operatorname{tr}(\boldsymbol{\Sigma}\!+\!\boldsymbol{\Sigma}')\!-\!2d\right]\quad\!$\refstepcounter{equation}(\theequation)\label{eq:normal4}\\
       \midrule
       Laplace
       %
       %& $\frac{2}{(2 \pi)^{d/2}\operatorname{det}(\boldsymbol{\Sigma})^{1/2}}\left(\frac{\boldsymbol{y}^T \boldsymbol{\Sigma}^{-1} \mathbf{y}}{2}\right)^{v / 2} \!\!K_v\left(\sqrt{2 \boldsymbol{y}^T \boldsymbol{\Sigma}^{-1} \boldsymbol{y}}\right)\;$\refstepcounter{equation}(\theequation)\label{eq:lap1}
       %
       %& $\frac{2}{(2 \pi)^{d/2}}\left(\frac{\lVert\mathbf{y}\rVert_2^2}{2}\right)^{v / 2} \!\!K_v\left(\sqrt{2}\lVert\mathbf{y}\rVert_2\right)\;$\refstepcounter{equation}(\theequation)\label{eq:lap2}
       %
       %& $-\frac{1}{2}\log \operatorname{det}\Sigma\!+\!\frac{d-2}{2} \big[\log\big(\operatorname{tr}(\Sigma)\big)\!-\!\log d\big]\;$\refstepcounter{equation}(\theequation)\label{eq:lap3}
       %
       %& $-\frac{1}{2}\log \operatorname{det}\Sigma\Sigma'\!+\!\frac{d-2}{2} \big[\log\big(\operatorname{tr}(\Sigma)\operatorname{tr}(\Sigma')\big)\!-\!2\log d\big]\;$\refstepcounter{equation}(\theequation)\label{eq:lap4}  \\
       %
       %
       & $\frac{2}{(2 \pi)^{d/2}\operatorname{det}(\boldsymbol{\Sigma})^{1/2}}\left(\frac{\boldsymbol{y}^T \boldsymbol{\Sigma}^{-1} \mathbf{y}}{2}\right)^{v / 2} \!\!K_v\left(\sqrt{2 \boldsymbol{y}^T \boldsymbol{\Sigma}^{-1} \boldsymbol{y}}\right)\;$\refstepcounter{equation}(\theequation)\label{eq:lap1}
       & $\frac{2}{(2 \pi)^{d/2}}\left(\frac{\lVert\mathbf{y}\rVert_2^2}{2}\right)^{v / 2} \!\!K_v\left(\sqrt{2}\lVert\mathbf{y}\rVert_2\right)\;$\refstepcounter{equation}(\theequation)\label{eq:lap2}
       & $-\frac{1}{2}\log \operatorname{det}(\Sigma)\!+\!\operatorname{tr}(\Sigma^{1/2})\!-\!d\;\;$\refstepcounter{equation}(\theequation)\label{eq:lap3}
       & $-\frac{1}{2}\log \operatorname{det}(\Sigma\Sigma')\!+\!\operatorname{tr}(\Sigma^{1/2}\!+\!\Sigma'^{1/2})\!-\!2d\;$\refstepcounter{equation}(\theequation)\label{eq:lap4}  \\
       \bottomrule[2pt] 
    \end{tabular}
    }
    %\vspace{-0.3cm}
        \caption{KL divergence $D_{KL}(p,q)$ between two multivariate  ({\em top})  Normal and  ({\em bottom}) symmetric Laplace  distributions with densities $p(\mathbf{y})$ and $q(\mathbf{y})$ and parameters $(\boldsymbol{0},\boldsymbol{\Sigma})$ and $(\boldsymbol{0},\boldsymbol{I})$, Moreover, $v\!=\!(2\!-\!d)/2$ and %$\rho\!=\!\lVert\boldsymbol{\mu}\rVert_2^2$, $\rho^*\!\!=\!\lVert\boldsymbol{\mu}\rVert_2^2\!+\!\lVert\boldsymbol{\mu}^*\rVert_2^2$, 
        $K_v(\cdot)$ is the second kind modified Bessel fun.}
    \label{tab:divergence}
    %\vspace{-0.2cm}
\end{table*}

%\vspace{-0.1cm}
\noindent
However, unlike the Euclidean space where points are constrained on the hypersphere (\eqref{eq:euc_align_uni}), the hyperbolic space is non-compact and has infinite volume. Thus, maximizing the pair-wise hyperbolic distance of pairs of points in the hyperbolic space via~\eqref{eq:hyper_uni} will move embedding towards the boundary of the Poincaré disk. It will also result in poor ``filling'' of the ambient space due to the infinite volume at the boundaries,  resulting in the  height and leaf collapse as in \figref{fig:hyper_tree_4} and a suboptimal performance (\tabref{tab:ab_study}). Moreover, computing pair-wise distance via \eqref{eq:hyper_uni} is computationally intense.

Due to the exponential growth of volume of hyperbolic manifold, %in $\mathbb{D}_c^d$, 
the desired  leaf- and height-level uniformity are required. Our investigation reveals that employing radially symmetric distributions in the tangent plane at $\mathbf{0}$, $T_{\mathbf{0}} \mathbb{D}^d_c$, yields desirable properties in the ambient space of the hyperbolic manifold. Specifically, we seek distributions that, when mapped to the hyperbolic space, exhibit the following key characteristics: 

\vspace{-0.2cm}
\begin{enumerate}[leftmargin=0.6cm]
    \item {\bf Preservation of radial symmetry.} The radial component of the mapped distribution depends only on the radial component of the original distribution in the tangent space.
    \item {\bf Uniform angular distribution.} The angular component of the mapped distribution remains uniform \wrt the angle. %, ensuring no directional bias.
    \item {\bf Outer shell isotropy.} As the radius increases toward the boundary of the Poincaré ball, the distribution density increases but remains uniform \wrt the angle, leading to what we call ``outer shell isotropy''.
\end{enumerate}

\vspace{-0.2cm}

% Our investigation reveals that employing an isotropic Normal distribution with zero mean, denoted as $\mathcal{N}\left(\mathbf{0}, \mathbf{I}\right)$, in the tangent plane at $\mathbf{0}$, $\mathcal{T}_{\mathbf{0}} \mathbb{D}^d_c$, yields satisfactory outer shell isotropy in the ambient space of the hyperbolic manifold, as depicted in \figref{fig:gaussain_dis_d}. 

These properties ensure that our method can effectively capture the hierarchical structure inherent in many datasets while preventing dimensional collapse. We mainly investigate this phenomenon by aligning a mutlivariate Normal distribution of features in the tangent space  (\figref{fig:gaussain_dis_d}) with a multivariate isotropic Normal distribution thanks to a closed-form KL divergence. %facilitating both theoretical analysis and practical implementation (as depicted in \figref{fig:gaussain_dis_d}).}
We  try  multivariate Normal and  Laplace based KL divergences from Table \ref{tab:divergence}.

Thus, we propose to align  in the tangent plane the feature distribution of learned representations with the isotropic Normal distribution. We gather a set of data vectors from the learned representations in one view, denoted as $\big\{ \boldsymbol{z_i} \big\}_{i=1}^N \in \mathbb{D}_c^d$, and calculate the mean and covariance matrix  as follows:
% \rex{if possible, we can refer to fig 5 and say which part ensures leaf uniformity and which part ensures height uniformity etc.}
%

\noindent
\begin{equation}
\setlength{\abovedisplayskip}{4pt}
    \!\!\!\!\!\!\!\!\boldsymbol{\mu}\!=\!\frac{1}{N} \sum_{i=1}^N \log _{\mathbf{0}}^c(\boldsymbol{z}_i), \; \boldsymbol{\Sigma}\!=\!\frac{1}{N} \sum_{i=1}^N\left[\log_{\mathbf{0}}^c(\boldsymbol{z}_i)\!-\!\boldsymbol{\mu}\right]^T\!\left[\log_{\mathbf{0}}^c(\boldsymbol{z}_i)\!-\!\boldsymbol{\mu}\right],
\end{equation}

\noindent
where $\boldsymbol{\mu} \in \mathbb{R}^d, \boldsymbol{\Sigma} \in \mathbb{S}^{d}_{++}$ (the space of symmetric positive definite matrices) and $d$ is the dimension  of vectors. For the second view of graph, we calculate by analogy $(\boldsymbol{\mu}', \boldsymbol{\Sigma}')$.

Let $\boldsymbol{y} = \log_{\boldsymbol{0}}^c({\boldsymbol{z}}) \in \mathbb{R}^{d}$ denote embedding in the he tangent plane at $\boldsymbol{0}$. We apply a Gaussian hypothesis to learned representations where $\boldsymbol{y}\!\sim\!\mathcal{N}(\boldsymbol{\mu}, \boldsymbol{\Sigma})$ and $p(\boldsymbol{y})$ is defined in \eqref{eq:normal}.

Then we let $\boldsymbol{y}'\!\!\sim\!\mathcal{N}(\boldsymbol{0}, \boldsymbol{I})$ and $q(\boldsymbol{y}')$ from \eqref{eq:normal2}. We calculate the KL divergence between multivariate Normal distributions $p$ \& $q$, yielding \eqref{eq:normal3} and the squared penalty of the mean (derived in Appendix \ref{app:kl-proof1}). Thus, $D(\mathbf{\Sigma},\boldsymbol{\mu})\!=\!D_{KL}\left(p, q \right)+\lVert\boldsymbol{\mu}\rVert_2^2$. For 
 cov. $\mathbf{\Sigma}$ \& $\mathbf{\Sigma}'$ (with  $\boldsymbol{\mu}$ \& $\boldsymbol{\mu}'$) of two graph views, 
we set $\mathcal{L}^{\mathbb{D}_c^d}_{U}$ in \eqref{eq:align_uni_loss} to the outer shell isotropy term $\mathcal{L}^{\mathbb{D}_c^d}_{U}\!=\!D(\mathbf{\Sigma}, \boldsymbol{\mu})+D(\mathbf{\Sigma}^{\prime}, \boldsymbol{\mu}^{\prime})$ as in \eqref{eq:normal4}.

For the KL divergence between symmetric multivariate Laplace distributions (derived in Appendix  \ref{app:lap}) with $p(\boldsymbol{y})$ and $q(\boldsymbol{y}')$ as in \eqref{eq:lap1} and (\ref{eq:lap2}), where $\boldsymbol{y}\!\sim\!\mathcal{L}_{ap}(\boldsymbol{\mu}, \boldsymbol{\Sigma})$ and $\boldsymbol{y}'\!\!\sim\!\mathcal{L}_{ap}(\boldsymbol{0}, \boldsymbol{I})$, we have $D(\mathbf{\Sigma},\boldsymbol{\mu})\!=\!D_{KL}\left(p, q \right)\!+\!\lVert\boldsymbol{\mu}\rVert_2^2$ based on \eqref{eq:lap3}, and $\mathcal{L}^{\mathbb{D}_c^d}_{U}=D(\mathbf{\Sigma}, \boldsymbol{\mu})+D(\mathbf{\Sigma}^{\prime}, \boldsymbol{\mu}^{\prime})$ based on \eqref{eq:lap4}.

\section{Theoretical Analysis}
Below we show that %zero-centered 
symmetric radial distributions %with density center at the origin monotonically decaying away from the center can 
can limit the dimensional collapse in hyperbolic space, \eg, imposing the zero-centered isotropic Normal distribution in the tangent space leads to the isotropic shell density in the ambient space which increases ERank (a known measure of DC) of the learned feature matrix.

%\vspace{-0.8cm}Analytical
\vspace{-0.1cm}
\subsection{Mapping of the Symmetric Distribution from the Tangent to the Ambient Space.}% of the Hyperbolic manifold.}

\begin{wrapfigure}{r}{0.45\linewidth}
\vspace{-0.9cm}
    \centering
\begin{subfigure}{0.32\linewidth}
        \centering
        \includegraphics[width=1.0\linewidth]{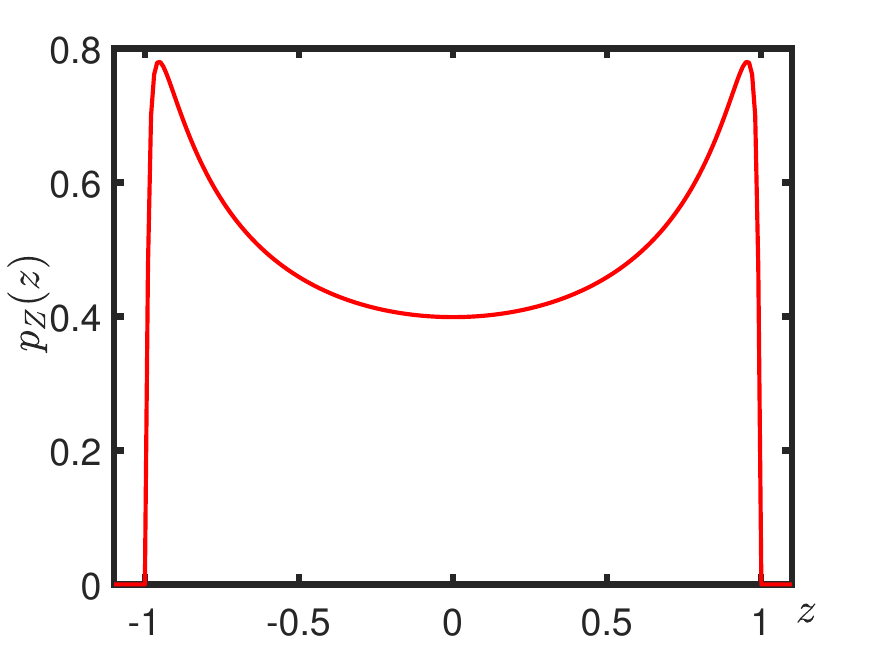} %Archive/NeurIPS2023/hyp1.png}
        \caption{1D ($\sigma\!=\!1, c\!=\!1$)}
        \label{fig:sim1d}
\end{subfigure}
% \hspace{5pt}
    \begin{subfigure}{0.32\linewidth}
        \centering
        \includegraphics[width=1.0\linewidth]{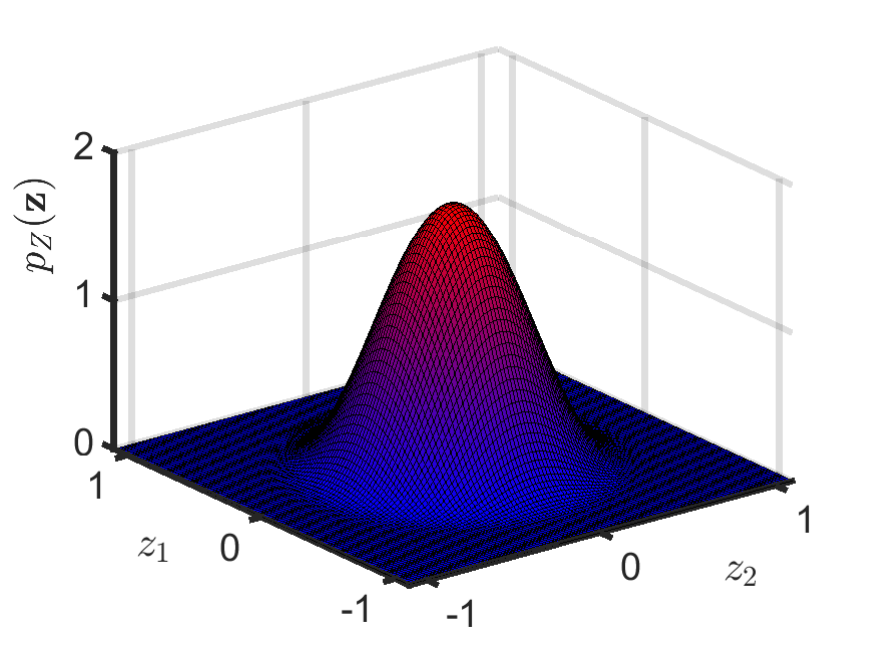} %Archive/NeurIPS2023/hyp1.png}
        \caption{2D ($\sigma\!=\!0.3, c\!=\!1$)}
        \label{fig:simadd1}
\end{subfigure}
% \hspace{5pt}
\begin{subfigure}{0.32\linewidth}
        \centering
        \includegraphics[width=1.0\linewidth]{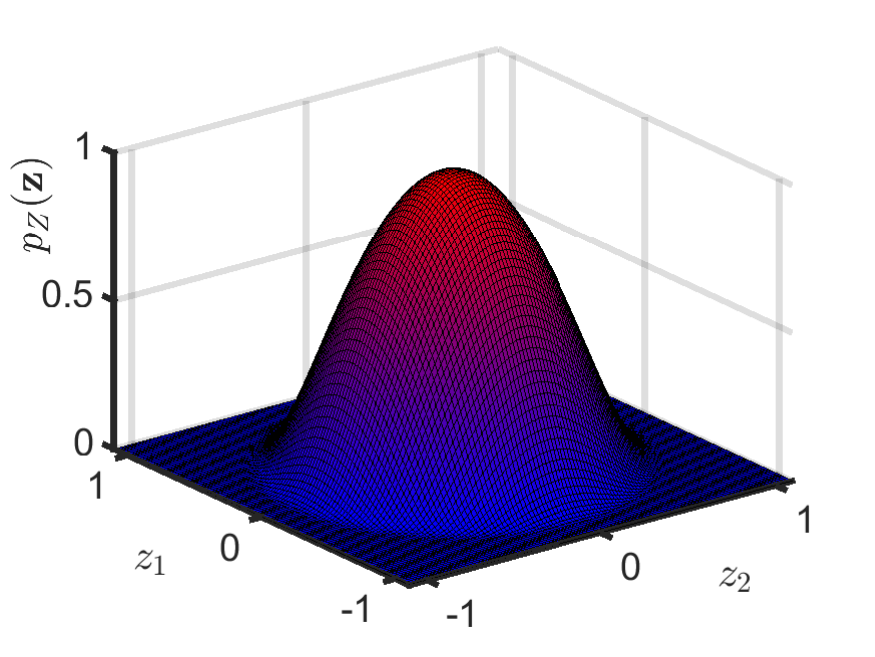} %Archive/NeurIPS2023/hyp1.png}
        \caption{2D ($\sigma\!=\!0.4, c\!=\!1$)}
        \label{fig:simadd1a}
\end{subfigure}
% \hspace{5pt}
%
\begin{subfigure}{0.32\linewidth}
        \centering
    \includegraphics[width=1.0\linewidth]{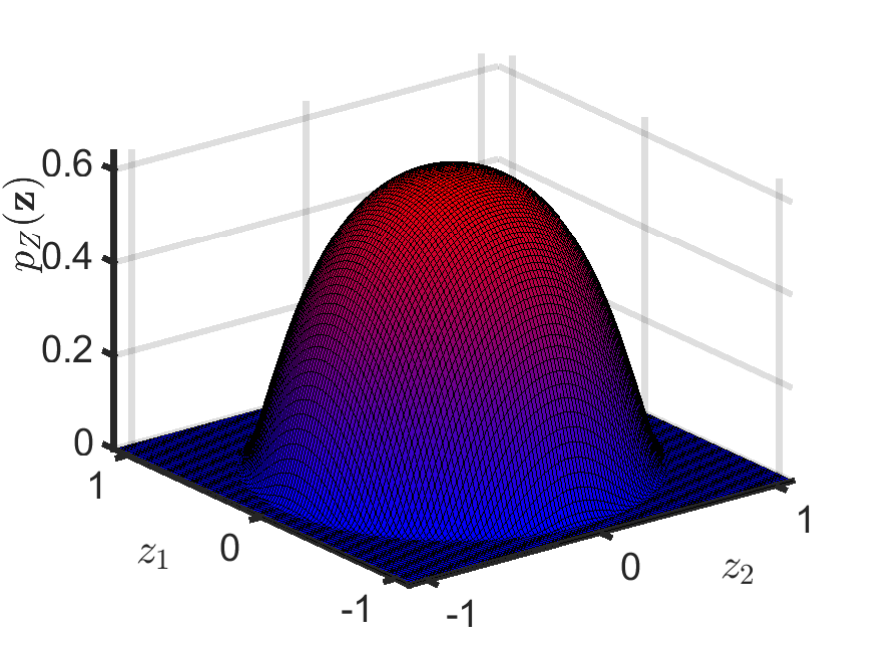}%{Archive/NeurIPS2023/hyp2.png}
        \caption{2D ($\sigma\!=\!0.5, c\!=\!1$)}
        \label{fig:simadd2}
\end{subfigure}
\begin{subfigure}{0.32\linewidth}
        \centering
    \includegraphics[width=1.0\linewidth]{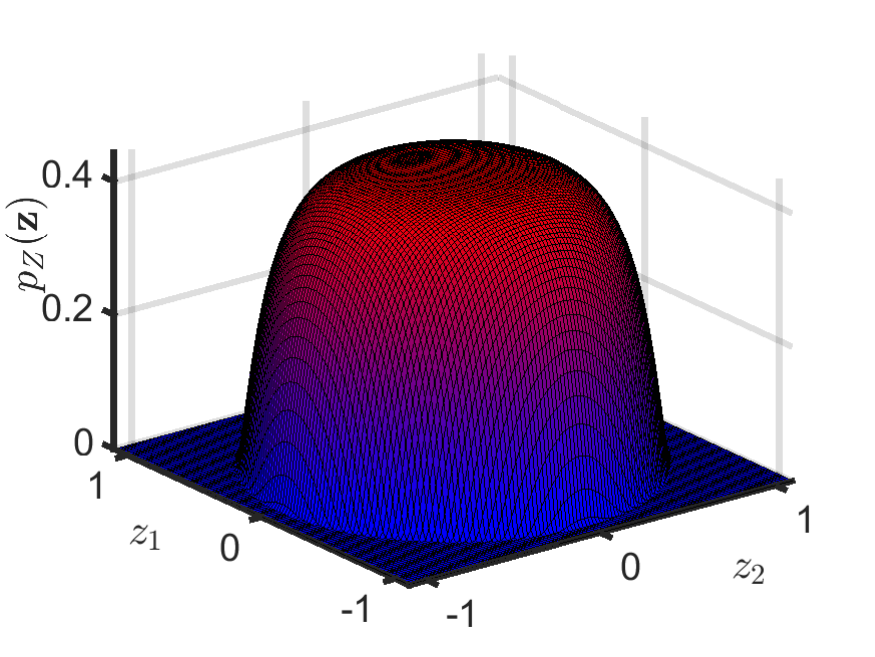}%{Archive/NeurIPS2023/hyp3.png}
        \caption{2D ($\sigma\!=\!0.6, c\!=\!1$)}
        \label{fig:simadd2b}
\end{subfigure}
\begin{subfigure}{0.32\linewidth}
        \centering
    \includegraphics[width=1.0\linewidth]{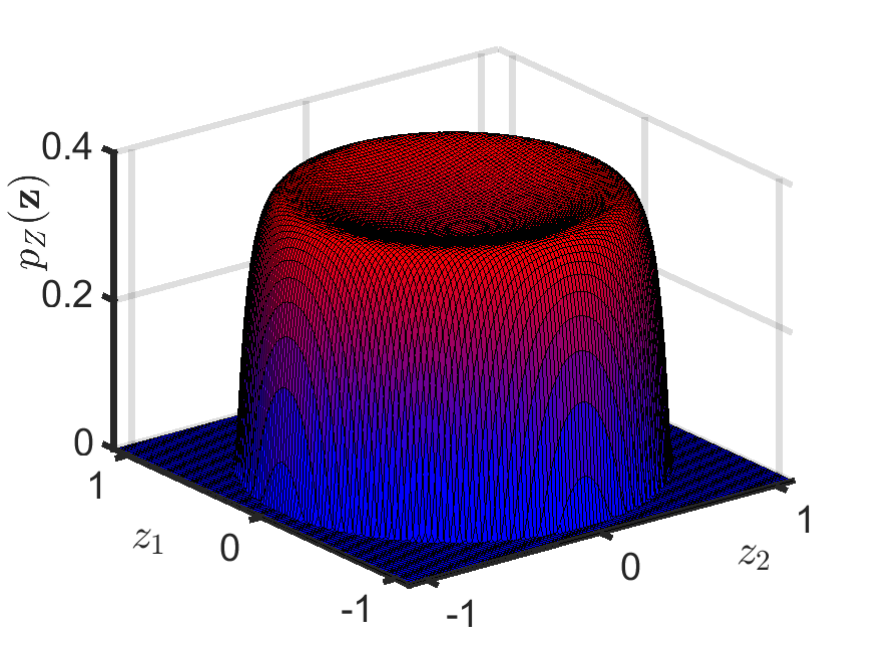}%{Archive/NeurIPS2023/hyp2.png}
        \caption{2D ($\sigma\!=\!.7, c\!=\!1$)}
        \label{fig:sim2da}
\end{subfigure}
\begin{subfigure}{0.32\linewidth}
        \centering
    \includegraphics[width=1.0\linewidth]{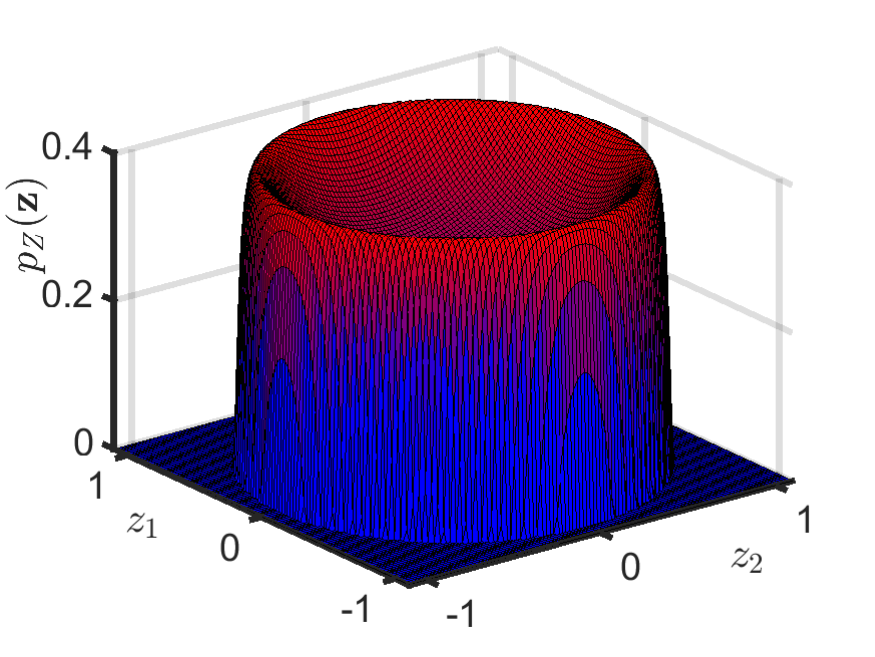}%{Archive/NeurIPS2023/hyp3.png}
        \caption{2D ($\sigma\!=\!.8, c\!=\!1$)}
        \label{fig:sim2db} 
\end{subfigure}
\begin{subfigure}{0.32\linewidth}
        \centering
    \includegraphics[width=1.0\linewidth]{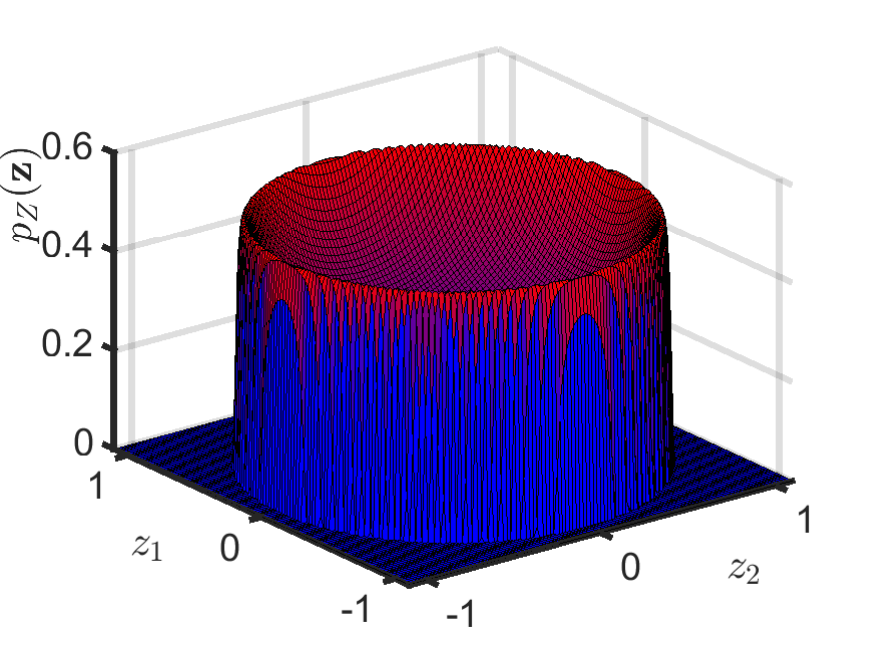}%{Archive/NeurIPS2023/hyp3.png}
        \caption{2D ($\sigma\!=\!0.9, c\!=\!1$)}
        \label{fig:simadd3b}  
\end{subfigure}
\begin{subfigure}{0.32\linewidth}
        \centering
    \includegraphics[width=1.0\linewidth]{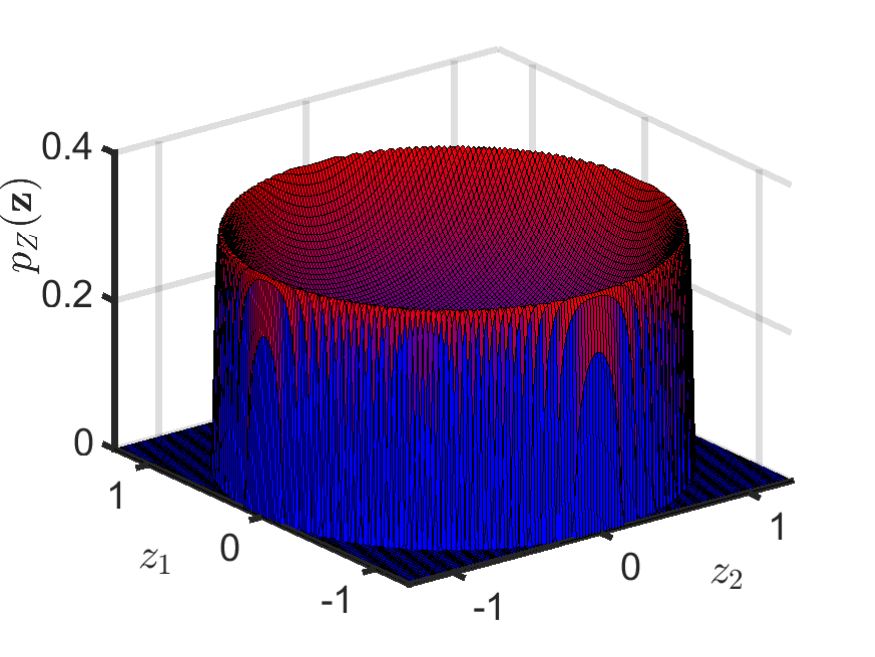}%{Archive/NeurIPS2023/hyp3.png}
        \caption{2D ($\sigma\!=\!1.2, c\!=\!0.6$)}
        \label{fig:simadd4}
\end{subfigure}
    %\vspace{-10px}
     \caption{The theoretical distribution $p_Z(\mathbf{z})$ of $\mathbf{z}$ (\eqref{eq:theor_dis}) in the ambient space of the hyperbolic manifold. As a sanity check, we evaluated the integral of each distribution (all equal one as expected).
    For small standard deviation $\sigma\!=\!0.3$ as in Fig. \ref{fig:simadd1} (and  in Fig. \ref{fig:simadd1a} and \ref{fig:simadd2}), the distribution resembles the Normal distribution and is likely to cause the height collapse. For medium standard deviation $\sigma\!=\!0.6$ as in Fig. \ref{fig:simadd2b}, one approximately recovers uniformity within the  Poincaré disk (undesired). Fig. \ref{fig:sim2db}, \ref{fig:simadd3b} and \ref{fig:simadd4} show that larger values of standard deviation lead to the maximum uniform density concentration on the outer shell of each distribution. Fig. \ref{fig:sim1d}: density (1D) for clear illustration of its peaks.
    }
        \label{sup:add}
    \vspace{-2.2cm}
\end{wrapfigure}

\begin{theorem}[Generalized Radial Distribution Mapping in Constant Curvature Spaces]
\label{th:rdm}
Let $(\mathcal{M}, g)$ be a Riemannian manifold of constant curvature $K$, and let $f: T_0\mathcal{M} \to \mathcal{M}$ be the exponential map at the origin. Consider a radial probability distribution $p_V(\mathbf{v})$ on $T_0\mathcal{M}$ that can be expressed in the general form:
$p_V(\mathbf{v}) = \phi(|\mathbf{v}|)$
where $\phi$ is any non-negative, integrable function that depends only on the norm of $\mathbf{v}$ and satisfies the normalization condition $\int_{T_0\mathcal{M}} \phi(|\mathbf{v}|) d\mathbf{v} = 1$.
Let $p_Z(\mathbf{z})$ be the distribution of $\mathbf{z} = f(\mathbf{v})$ in $\mathcal{M}$. Then the radial component of $p_Z(\mathbf{z})$ depends only on the radial component of $p_V(\mathbf{v})$.
The angular component of $p_Z(\mathbf{z})$ is uniform on the unit sphere in $T_z\mathcal{M}$.
$p_Z(\mathbf{z})$ exhibits ``outer shell isotropy'' in $\mathcal{M}$, defined as:
$\lim_{|\mathbf{z}| \to R} \frac{p_Z(\mathbf{z}_1)}{p_Z(\mathbf{z}_2)} = 1$
for any $\mathbf{z}_1, \mathbf{z}_2$ with $|\mathbf{z}_1| = |\mathbf{z}_2|$, where $R\!\geq\!0$ is the maximum radius in $\mathcal{M}$ (infinite for $K \leq 0$, finite for $K > 0$).
\end{theorem}

\vspace{-0.3cm}
\begin{proof}
See Appendix \ref{app:proof1}. 
\end{proof}
Theorem \ref{th:rdm} is a general statement about the behavior of radial distributions when mapped to constant curvature spaces. Intuitively, 1) as the uniform shell density coincides with the uniform tree nodes density at a given depth, we need that property. 2) % this provides even more flexibility in designing and analyzing methods for hyperbolic graph neural networks and other applications in constant curvature spaces.
We also desire density to increase exponentially \wrt the radius, proportionally to the total count of nodes $1\!+\!\rho^1+\ldots+\rho^{l-1}$ at level $l$ in a balanced tree with $\rho$ children per parent.
The multivariate Normal distribution below and the symmetric multivariate Laplace meet these criteria.

\begin{theorem}
\label{th:thth}
The Normal distribution can be mapped into the ambient space of the Hyperbolic manifold by the change of variable. Let the transformation variable be $f(\boldsymbol{v})=\exp _{\mathbf{0}}^c(\boldsymbol{v})=\tanh (\sqrt{c}\|\boldsymbol{v}\|) \frac{\boldsymbol{v}}{\sqrt{c}\|\boldsymbol{v}\|_2}$. Then by mapping the Normal distribution from  the tangent plane at $\mathbf{0}$ via $f(\boldsymbol{v})$ to the ambient space, we obtain    distribution:
\begin{equation}
p_Z(\mathbf{z})=0.5\,\delta\big(p_\mathcal{N}(\log_{\mathbf{0}}^c(\boldsymbol{z}); \boldsymbol{\mu}, \mathbf{\Sigma})\cdot\lambda^c_{\mathbf{z}} \,g^{d-1}(\mathbf{z})\big),
\label{eq:theor_dis}
\end{equation}
where $p_\mathcal{N}(\cdot)$ is the Probability Density Function (PDF)  of the Normal distribution, $\lambda^c_{\mathbf{z}}=\frac{2}{1-c\|\mathbf{z}\|_2^2}$ is the so-called conformal factor, $g(\mathbf{z})=\frac{1}{\sqrt{c}\|\mathbf{z}\|_2}\tanh^{-1}(\sqrt{c}\|\mathbf{z}\|_2),$   $\delta(v)\!=\!v$ if Imaginary$(v)\!=\!0$ else $\delta(v)\!=\!0$.
\begin{proof}
See Appendix \ref{app:proof1}. 
%It directly follow from the transformation of random variables. Specifically, $p_Z(\mathbf{z})=p_\mathcal{N}(f^{-1}(\mathbf{z}))\cdot\det\big(\mathbf{J}(f^{-1}(\mathbf{z})) \big)$. Notice that for $f(\boldsymbol{v})=\exp _{\mathbf{0}}^c(\boldsymbol{v})$ the inverse is logarithmic map $f^{-1}(\boldsymbol{z})=\log _{\mathbf{0}}^c(\boldsymbol{z})=\frac{1}{\sqrt{c}\|\mathbf{z}\|_2}\tanh^{-1}(\sqrt{c}\|\mathbf{z}\|_2)\frac{\boldsymbol{z}}{\sqrt{c}\|\boldsymbol{z}\|_2}$. The main difficulty lies with computing the Jacobian $\mathbf{J}(f^{-1}(\mathbf{z}))$ and its determinant $\det\big(\mathbf{J}(f^{-1}(\mathbf{z})) \big)$, which (after crunching some maths) turns out to enjoy a simple analytical form $0.5\,\lambda^c_{\mathbf{z}} \,g^{d-1}(\mathbf{z})$.
\end{proof}
\vspace{-10px}
\end{theorem}

\begin{table*}[t]
\vspace{-0.5cm}
\setlength\tabcolsep{8pt}
\fontsize{8}{7}\selectfont
\centering
\caption{Comparison with various node classification models. \rh~marks the methods that are trained in the supervised manner and \bd~marks the methods that are trained in the self-supervised manner.}
\begin{tabular}{clccccc}
\toprule[2pt] 
\textbf{Space}       
& \textbf{Method}             & \textbf{Disease}    &\textbf{Airport}     &\textbf{PubMed}     &\textbf{CiteSeer}  & \textbf{Cora}     \\\midrule 
\multirow{6}{*}{ Euclidean }
& \rh~\textbf{GCN}            & \ms{69.70}{0.40}     &  \ms{81.40}{0.60}   &\ms{81.10}{0.20}    &\ms{70.41}{0.52}    &\ms{81.30}{0.30}   \\
& \rh~\textbf{GAT}            & \ms{70.40}{0.40}     &  \ms{81.50}{0.30}   &\ms{82.00}{0.30}    &\ms{70.14}{0.38}    &\ms{83.00}{0.70}   \\
& \rh~\textbf{SAGE}           & \ms{69.10}{0.60}     &  \ms{82.10}{0.50}   &\ms{80.40}{2.20}    &\ms{69.91}{1.38}    &\ms{77.90}{2.40}   \\
& \bd~\textbf{GRACE}          & \ms{69.61}{0.49}     &  \ms{82.79}{0.40}   &\ms{83.51}{0.37}    &\ms{71.42}{0.64}    &\ms{81.13}{0.44}   \\
& \bd~\textbf{COSTA}          & \ms{67.12}{0.39}     &  \ms{81.19}{0.40}   &\ms{84.31}{0.37}    &\ms{70.77}{0.24}    &\ms{82.14}{0.62}   \\\midrule
\multirow{8}{*}{ Hyperbolic }     
& \rh~\textbf{HNN}            & \ms{75.18}{0.25}     & \ms{80.59}{0.46}    &\ms{76.88}{0.43}    &\ms{59.50}{1.28}    &\ms{54.76}{0.61}   \\ 
& \rh~\textbf{HGNN}           & \ms{81.27}{1.53}     & \ms{84.71}{0.98}    &\ms{80.13}{0.82}    &\ms{69.99}{1.00}    &\ms{78.26}{1.19}   \\
& \rh~\textbf{HGCN}           & \ms{88.16}{0.76}     & \ms{89.26}{1.27}    &\ms{82.53}{0.63}    &\ms{70.34}{0.59}    &\ms{78.03}{0.98}   \\
& \rh~\textbf{HGAT}           & \ms{90.30}{0.62}     & \ms{89.62}{1.23}    &\ms{81.42}{0.36}    &\ms{70.64}{0.30}    &\ms{78.32}{1.39}   \\ 
& \rh~\textbf{HIE}            & \ms{94.50}{0.43}     & \ms{93.55}{1.11}    &\ms{85.14}{0.38}   &\ms{73.43}{0.35}   &\ms{83.47}{0.34}  \\ 
& \bd~\textbf{HGCL}           & \ms{93.42}{0.82}     &\ms{92.35}{1.51}     &\ms{83.14}{0.58}    &\ms{72.11}{0.64}    &\ms{82.37}{0.47}   \\
& \bd~\textbf{HyperGCL}$_{\text{Laplace}}$       & \ums{94.89}{0.45}      & \ubms{94.81}{0.74} &\ums{85.49}{0.22}  &\ums{74.37}{0.42}  &\ums{84.67}{0.50}  \\ 
& \bd~\textbf{HyperGCL}$_{\text{Normal}}$        & \ubms{95.30}{0.43}      & \ums{94.55}{0.81} &\ubms{85.56}{0.21}  &\ubms{74.43}{0.35}  &\ubms{84.77}{0.46}  \\ 
\bottomrule[2pt]
\end{tabular}
\label{tab:node_class_result}
% \vspace{-0.6cm}
\end{table*}

Figure \ref{sup:add} proves that the analytical distribution in \eqref{eq:theor_dis} matches  simulations in Fig. \ref{fig:gaussain_dis}. We  also verified that $\int_{\mathbf{z}\in\texttt{support}(Z)}p_Z(\mathbf{z})d\mathbf{z}\!=\!1$.$\!\!\!$. Fig. \ref{sup:add} also shows the impact of variance $\sigma$ and curvature $c$ on the distribution imposed in the ambient space of the manifold. In essence we simulate several ``good'' cases (Fig. \ref{fig:sim1d}, \ref{fig:sim2da}, \ref{fig:sim2db}, \ref{fig:simadd3b} \& \ref{fig:simadd4}) and the depth collapse (\ref{fig:simadd1}, \ref{fig:simadd1a}, \ref{fig:simadd2} \& \ref{fig:simadd2b}). Figure \ref{sup:add2} illustrates the impact of non-zero mean, leading to the leaf collapse.

\begin{tcolorbox}[width=1\linewidth, colframe=blackish, colback=beaublue, boxsep=0mm, arc=2mm, left=2mm, right=2mm, top=1mm, bottom=1mm]
\textbf{The Wrapped Normal (Theorem \ref{th:thth}) is Radial.} As $g(\mathbf{z})$ and $\lambda_z^c$ in \eqref{eq:theor_dis} depend on $r\!=\!\lVert\mathbf{z}\rVert_2$, and $p_\mathcal{N}(\log_{\mathbf{0}}^c(\boldsymbol{z}); \boldsymbol{\mu}, \mathbf{\Sigma})\!=\!p_\mathcal{N}(rg(\mathbf{z})\mathbf{z}; \boldsymbol{0}, \mathbf{I})$ for isotropic $(\mathbf{\Sigma},\mathbf{\mu})\!=\!(\mathbf{I},\mathbf{0})$ can be expressed as  $p(r)\!=\!\frac{1}{2\pi}^{d/2}\exp(-\frac{1}{2}r'^2)$ where $r'^2\!=\!\lVert rg(\mathbf{z})\mathbf{z}\rVert_2^2$, that means that the density in \eqref{eq:theor_dis} depends only on a radius, as per notion of ``outer shell isotropy'' in Theorem \ref{th:rdm}.
We also notice from Fig. \ref{fig:sim1d} that $p(r)\!\propto\!1\!+\!\rho^1+\ldots+\rho^{l-1}$ at level $l$ which satisfies the exponential density grow \wrt tree level $l\!\propto\!r$ \& node degree $\rho$.
\end{tcolorbox}

\begin{tcolorbox}[width=1\linewidth, colframe=blackish, colback=beaublue, boxsep=0mm, arc=2mm, left=2mm, right=2mm, top=1mm, bottom=1mm]
\textbf{Tree depth ($1\!+\!\varepsilon$ distortion Delaunay embedding).} If $\epsilon\!=\!0$ in \eqref{eq:poj}, no doubt the geodesic distance between the origin $\mathbf{0}$ and $\mathbf{z}$ on the Poincaré boundary is  $\lim_{\lVert \mathbf{z}\rVert_2\rightarrow 1/\sqrt{c}}D_c(\mathbf{0},\mathbf{z})\!=\!+\infty$. Thus, an embedded tree can be infinitely deep unless the feature space on the boundary is not occupied by features. However, $\epsilon\!>\!0$ \& the outer shell isotropy radius $0\!\ll\!r\!\leq\!(1\!-\!\epsilon)\frac{1}{\sqrt{c}}$ give relation $\infty\!>\!\frac{2}{\sqrt{c}}\tanh^{-1}(\sqrt{c}r)\!\geq\!le\!>\!l\nu\frac{1+\varepsilon}{\varepsilon}$ for the number of tree levels $l$, edges length $e\!>\!\nu\frac{1+\varepsilon}{\varepsilon}$ (``Alg.: Distortion embedding'', page 11 \cite{sarkar2011low}), 
 the distortion factor $1+\varepsilon$ 
and $\nu\!=\!-\log\tan(\beta/2)$ for a cone separation angle $\beta\!<\!\pi/\rho$ for node degree $\rho$.
\end{tcolorbox}

%the integral of the  distribution in the support region equals 1.

% \subsection{Universal Distribution Beyond Normal Distribution} 
% ToDo

\subsection{Equivalence of Effective Rank in the Tangent Space}

\begin{wrapfigure}{r}{0.45\linewidth}
\vspace{-0.9cm}
   \centering
\begin{subfigure}{0.48\linewidth}
        \centering
        \includegraphics[width=1.0\linewidth]{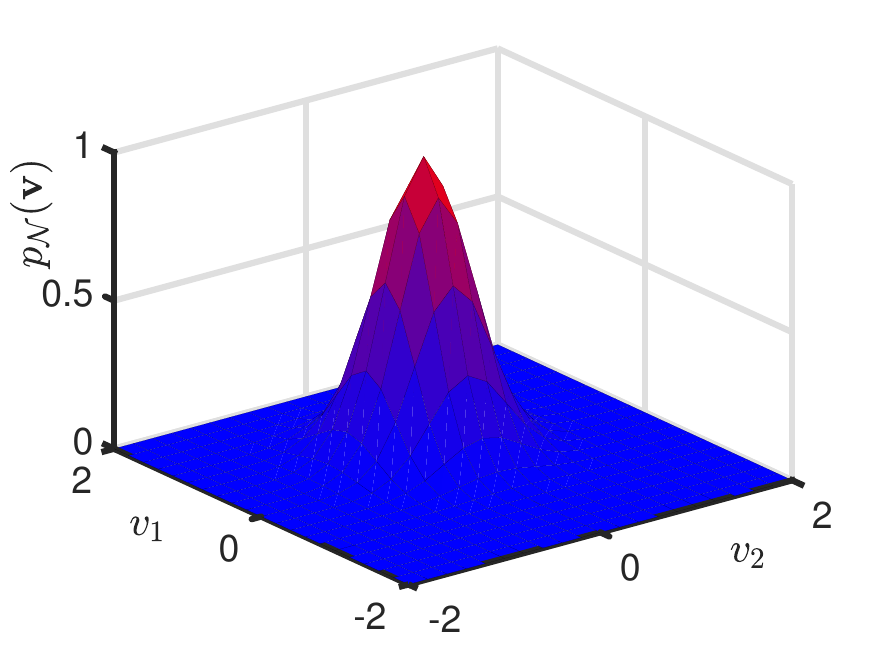} %Archive/NeurIPS2023/hyp1.png}
        \caption{2D ($\sigma\!=\!0.4, \boldsymbol{\mu}\!=\![0.3, 0.3]$)}
        \label{fig:g1}
\end{subfigure}
% \hspace{5pt}
\begin{subfigure}{0.48\linewidth}
        \centering
        \includegraphics[width=1.0\linewidth]{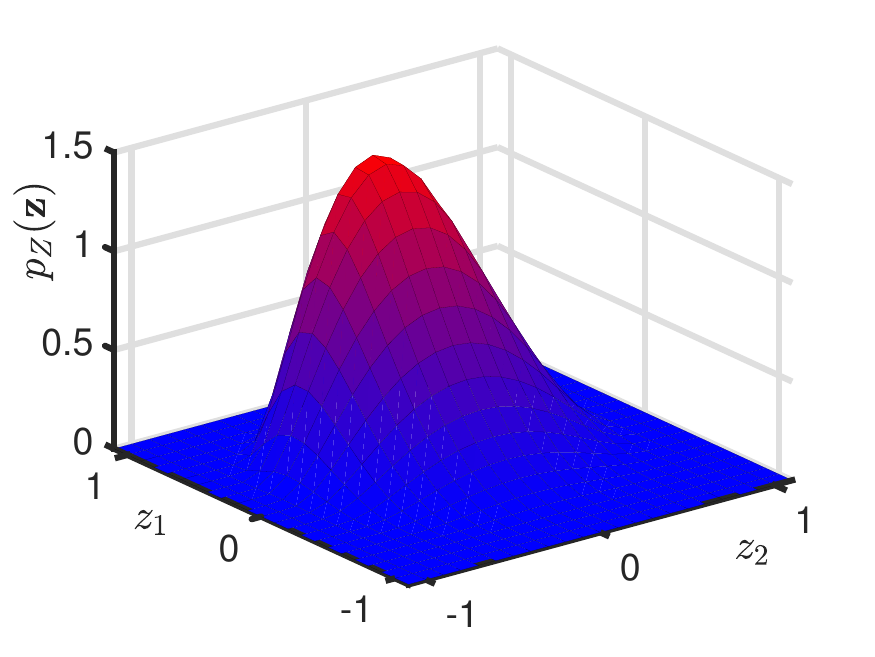}%{Archive/NeurIPS2023/hyp3.png}
        \caption{2D ($\sigma\!=\!0.3, \boldsymbol{\mu}\!=\![0.3, 0.3]$)}
        \label{fig:h1}
\end{subfigure}
\begin{subfigure}{0.48\linewidth}
        \centering
        \includegraphics[width=1.0\linewidth]{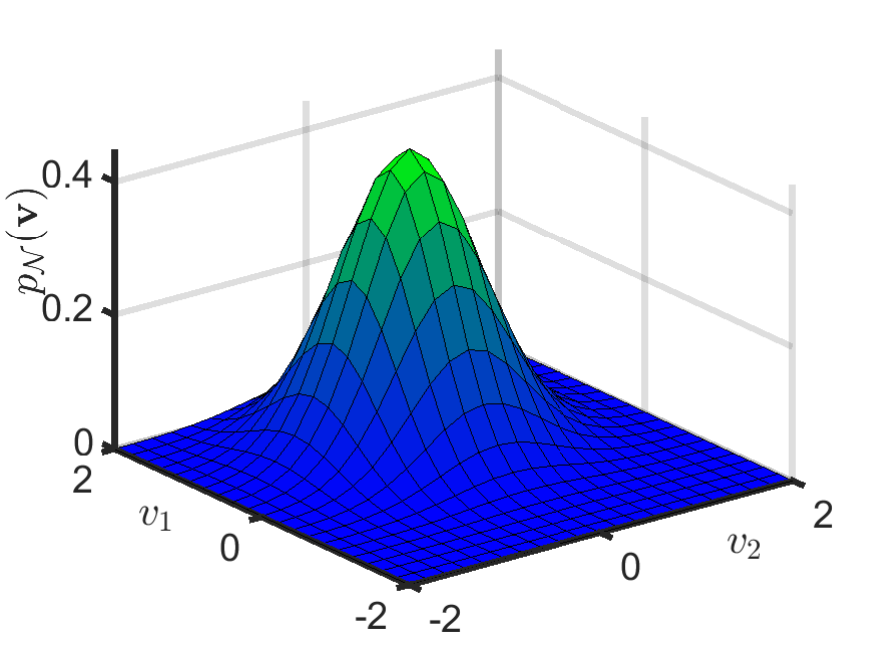} %Archive/NeurIPS2023/hyp1.png}
        \caption{2D ($\sigma\!=\!0.6, \boldsymbol{\mu}\!=\![0.6, 0.6]$)}
        \label{fig:g2}
\end{subfigure}
\begin{subfigure}{0.48\linewidth}
        \centering
        \includegraphics[width=1.0\linewidth]{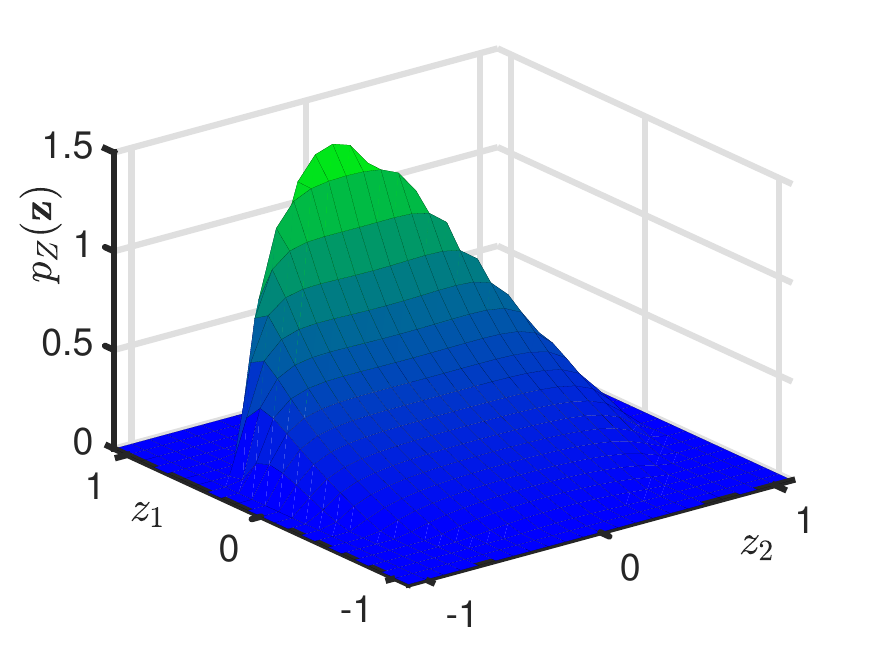}%{Archive/NeurIPS2023/hyp3.png}
        \caption{2D ($\sigma\!=\!0.6, \boldsymbol{\mu}\!=\![0.3, 0.66]$)}
        \label{fig:h2}
\end{subfigure}
\begin{subfigure}{0.48\linewidth}
        \centering
        \includegraphics[width=1.0\linewidth]{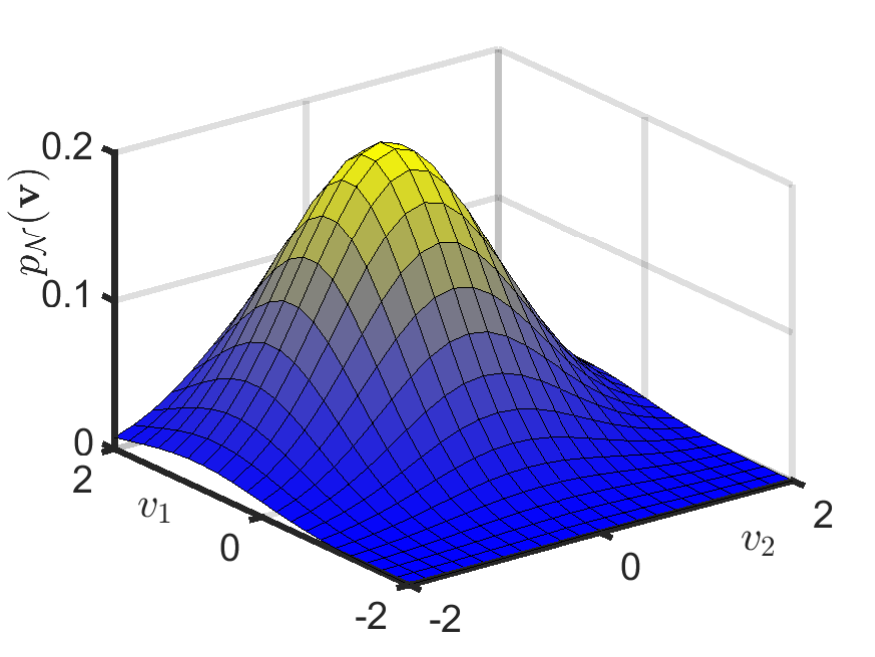}%{Archive/NeurIPS2023/hyp2.png}
        \caption{2D ($\sigma\!=\!0.9, \boldsymbol{\mu}\!=\![0.9, 0.9]$)}
        \label{fig:g3}
\end{subfigure}
\begin{subfigure}{0.48\linewidth}
        \centering
        \includegraphics[width=1.0\linewidth]{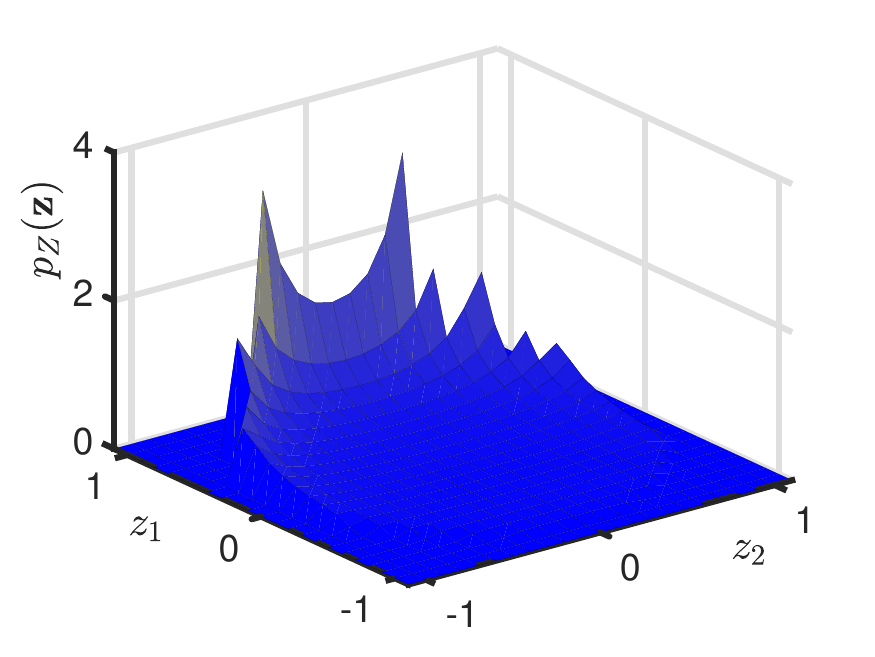}%{Archive/NeurIPS2023/hyp3.png}
        \caption{2D ($\sigma\!=\!0.9, \boldsymbol{\mu}\!=\![0.9, 0.9]$)}
        \label{fig:h3}
\end{subfigure}
%\hfill
%\vspace{-0.3cm}
    \caption{({\em left}) The Normal distribution $p_\mathcal{N}(\mathbf{v})$ of $\mathbf{v}$ imposed in the tangent space, and ({\em right}) the theoretical distribution $p_Z(\mathbf{z})$ of $\mathbf{z}$ (\eqref{eq:theor_dis}) in the ambient space of the hyperbolic manifold.
    We vary the standard deviation $\sigma$ and the mean $\boldsymbol{\mu}$, thus simulating the mean shift leading to the  leaf collapse.
    }
    \label{sup:add2}
    %\vspace{-10pt}
   % \vspace{-1.0cm}
   \vspace{-1.2cm}
\end{wrapfigure}

Below we show that minimizing $\mathcal{L}^{\mathbb{D}_c^d}_{U}$ increases the Effective Rank (ERank) which alleviates the dimensional collapse. The ERank in the tangent space   measures the effective dimension of embedding output by the encoder $f_{\Theta}(\cdot)$ in the tangent space. The higher effective rank denotes the lower degree of dimensional collapse. As ERank in the ambient space of the hyperbolic manifold is correlated (by a non-linear mapping) with the ERank in the tangent space, we can directly show how imposing the zero-centered isotropic Normal distribution in the tangent space increases the ERank.  

To show the equivalency between optimizing negative ERank and our loss $\mathcal{L}^{\mathbb{D}_c^d}_{U}$, we  show: i) in Theorem \ref{theorem:convex} that they both are convex, ii) in Theorem \ref{theorem:opt} that they both attain optimum for the uniform spectrum. Firstly, we define the Effective Rank (ERank).
\begin{definition}[Effective Rank (ERank)]
    Let  matrix $\boldsymbol{X} \!\in \!\mathbb{R}^{m \times n}\!$ with  $\boldsymbol{X}\!=\!\boldsymbol{U} \boldsymbol{\Sigma} \boldsymbol{V}^T\!$ as its singular value decomposition, where $\boldsymbol{\Sigma}$ is a diagonal matrix with singular values $\sigma_1\! \geq\!  \cdots\! \geq\!\sigma_Q\! \geq \!0$ with $Q\!=\!\min(m, n)$. The distribution of singular values is defined as the normalized form $p_i=\sigma_i / \sum_{k=1}^Q\left|\sigma_k\right|$. The effective rank of the matrix $\boldsymbol{X}$, %denoted as {\em Erank}$(\boldsymbol{X})$, 
    is defined as $\operatorname{ERank}(\boldsymbol{X})=\exp \left(H\left(p_1, p_2, \cdots, p_Q\right)\right)$, where $H\left(p_1, p_2, \cdots, p_Q\right)$ is the Shannon entropy  $H\left(p_1, p_2, \cdots, p_Q\right)=-\sum_{k=1}^Q p_k \log p_k$.
\end{definition}
% \begin{theorem}
% \label{th:er}
% The Effective Rank of $\mathbf{\Sigma}$ is lower-bounded by $-D(\mathbf{\Sigma}, \boldsymbol{\mu})$ as:
% \begin{equation}
%    -D(\mathbf{\Sigma}, \boldsymbol{\mu}) \leq \log \left[\operatorname{Erank}(\mathbf{\Sigma})\right] + const
% \end{equation}
% Let $\{\lambda\}_{i=1}^d$ be the eigenvalues of $\mathbf{\Sigma}$, the equality is hold when $\lambda_i = \lambda_j$ for all $i, j$.
% \label{lemma:erank_bound}
% \end{theorem}
% \begin{proof}
% % \vspace{-1cm}
%     Proof of Lemma \ref{lemma:erank_bound} is in Appendix \ref{proof:erank_bound}.
% \end{proof}

\begin{theorem}[Convexity]Define a loss function $\mathcal{L}_{ERank}(\boldsymbol{\Sigma}) = - \text{ERank}(\boldsymbol{\Sigma})$ on covariance $\boldsymbol{\Sigma}$. Both $\mathcal{L}_{ERank}(\boldsymbol{\Sigma})$ and our loss function $D(\boldsymbol{\Sigma}, \boldsymbol{\mu})$ are convex functions \wrt eigenvalues $\sigma$ of $\boldsymbol{\Sigma}$. 
\label{theorem:convex}
\end{theorem}

\begin{theorem}[Optima] 
Without loss of generality, assume $\operatorname{tr}(\boldsymbol{\Sigma})$ is constrained by some  $d'$. Then both $D(\boldsymbol{\Sigma}, \boldsymbol{0})$ and $\mathcal{L}_{ERank}(\boldsymbol{\Sigma})$ enjoy the same optimum point such that eigenvalues $\sigma_1\!=\!\sigma_2\!=\ldots=\!\sigma_d\!=\!d'/d$.
%Consider $\Sigma_1^{*}$ as the optimal solution for $\text{argmin}_\Sigma \mathcal{L}_{ld}(\Sigma)$ and $\Sigma_2^{*}$ as the optimum for $\text{argmin}{\Sigma}\mathcal{L}_{ld}(\Sigma)$. It follows that $\Sigma_1^* = \Sigma_2^* = \text{diag}(\sigma^*_1, \sigma^*_2, \cdots, \sigma^*_n)$, where $\sigma^*_1 = \sigma^*_2 \cdots  \sigma^*_n$. 
 \label{theorem:opt}
\end{theorem}

%\vspace{-0.3cm}
\begin{tcolorbox}[width=1\linewidth, colframe=blackish, colback=beaublue, boxsep=0mm, arc=2mm, left=2mm, right=2mm, top=1mm, bottom=0.5mm]
\textbf{Minimizing our $\mathcal{L}^{\mathbb{D}_c^d}_{U}$ equals Maximizing ERank.} 
 Theorems \ref{theorem:convex} and \ref{theorem:opt} show that both $D(\boldsymbol{\Sigma}, \boldsymbol{0})$ and $\mathcal{L}_{ERank}(\boldsymbol{\Sigma})$ are convex and achieve the optimum for the uniform spectrum $\sigma_1\!=\!\sigma_2\!=\ldots=\!\sigma_d\!=\!d'/d$ (in case of $D(\boldsymbol{\Sigma}, \boldsymbol{0})$, we have $d'\!=\!d$). \textbf{The importance of the above optimum:} our $\mathcal{L}^{\mathbb{D}_c^d}_{U}$ in \eqref{eq:normal4} encourages uniformity of the spectrum, thus increases ERank (low ERank  is known to indicate the dimensional collapse).
 %
 %This suggests that the proposed loss function encourages a learning dynamic that moves towards maximizing effective rank. 
 %
\end{tcolorbox}

 %\vspace{-0.1cm}
 %Thus, by minimizing $\mathcal{L}^{\mathbb{D}_c^d}_{U}$, a higher ERank is achieved, indicating lower degree of dimensional collapse. 
 Notice the use of ERank directly is impractical: calculating the entropy of the eigenvalues %of $\Sigma$ 
 via SVD is time-consuming and numerically unstable in backpropagation due to occurrence of so-called non-simple eigenvalues ($\sigma_i\!=\!\sigma_j, i\!\neq\!j$). In contrast,  $\log\operatorname{det}$ in \eqref{eq:normal4}  computes determinants by the Cholesky factorization~\cite{golub2013matrix} that has stable gradients. Appropriately parameterized LogDet divergence is an upper bound of the negative log of ERank \cite{gap_dc}.
% \end{tcolorbox}

% \subsection{Connection to the Decorrelation in Contrastive Learning}
% In this section, we link $\mathcal{L}_{U}^{\mathcal{T}}$ to the Decorrelation loss

\begin{table*}[!tp]
\caption{Comparison with various competing models. \os~~denotes methods that are designed in the Euclidean space and \cc~~denotes methods that are designed in the hyperbolic space.~\label{tab:hyper_cf}}
% \resizebox{\textwidth}{!}{%
\setlength\tabcolsep{4pt}
\fontsize{8}{7}\selectfont
\centering
\begin{tabular}{clccccccc}
\toprule[2pt]
\multicolumn{2}{c}{\textbf{Datasets}}        & \os\textbf{NGCF}   & \os\textbf{LGCN} & \cc\textbf{HAE}    & \cc\textbf{HAVE}   & \cc\textbf{HGCF}               & \cc\textbf{HRCF}             & \cc\textbf{HyperGCL}$_{\text{Normal}}$ \\ \midrule 
\multirow{4}{*}{\textbf{Amazon-CD}}   
                             & R@10  & 0.0758 & 0.0929   & 0.0666 & 0.0781 & 0.0962 & \um{0.1003} &\ubm{0.1069}\\
                             & R@20  & 0.1150 & 0.1404   & 0.0963 & 0.1147 & 0.1455 & \um{0.1503} &\ubm{0.1573}\\ 
                             & N@10  & 0.0591 & 0.0726   & 0.0565 & 0.0629 & 0.0751 & \um{0.0785} &\ubm{0.0825}\\
                             & N@20  & 0.0718 & 0.0881   & 0.0657 & 0.0749 & 0.0909 & \um{0.0947} &\ubm{0.1043}\\ \midrule
\multirow{4}{*}{\textbf{Amazon-Book}}
                             & R@10  & 0.0658 & 0.0799   & 0.0634 & 0.0774 & 0.0867 & \um{0.0900} &\ubm{0.0973}\\
                             & R@20  & 0.1050 & 0.1248   & 0.0912 & 0.1125 & 0.1318 & \um{0.1364} &\ubm{0.1489}\\ 
                             & N@10  & 0.0655 & 0.0780   & 0.0709 & 0.0778 & 0.0869 & \um{0.0902} &\ubm{0.0982}\\
                             & N@20  & 0.0791 & 0.0938   & 0.0789 & 0.0901 & 0.1022 & \um{0.1060} &\ubm{0.1060}\\ \midrule
\multirow{4}{*}{\textbf{Yelp2020}}    
                             & R@10  & 0.0458 & 0.0522   & 0.0360 & 0.0421 & 0.0527 & \um{0.0537} &\ubm{0.0587}\\
                             & R@20  & 0.0764 & 0.0866   & 0.0588 & 0.0691 & 0.0884 & \um{0.0898} &\ubm{0.0950}\\ 
                             & N@10  & 0.0405 & 0.0461   & 0.0331 & 0.0371 & 0.0458 & \um{0.0468} &\ubm{0.0508}\\
                             & N@20  & 0.0513 & 0.0582   & 0.0409 & 0.0465 & 0.0585 & \um{0.0594} &\ubm{0.0639}\\ \bottomrule[2pt]
\end{tabular}%
% }
%\vspace{-0.2cm}
\end{table*}

\section{Experiments}
% Below, we conduct extensive experiments to verify our proposed method in different tasks.
\subsection{Results on Graph Representation Learning}\label{sec:result_graph}
\noindent\textbf{Datasets.}
We use the citation networks \textbf{Cora}, \textbf{CiteSeer} \& \textbf{PubMed} whose %are standard benchmarks describing citation networks where
nodes represent scientific papers~\cite{kipf2016semi}. We also use the \textbf{Disease} dataset %which is
%from the SIR disease spreading model~\cite{anderson1991infectious}, where the
whose node labels  indicate ``infected'' nodes. \textbf{Airport}'s nodes/edges  represent airports/routes from OpenFlights.org~\cite{zhang2018link}. In the node classification (NC), we use 70/15/15\% splits for the Airport dataset, 30/10/60\% splits for Disease. We use standard splits~\cite{kipf2016semi} with 20 train examples per class for Cora, CiteSeer and PubMed. As in \cite{velivckovic2017graph}, we evaluate node classification by measuring accuracy.

\vspace{0.1cm}
\noindent\textbf{Setting and Baselines.} 
% We consider Euclidean embeddings and hyperbolic embeddings. We compare our model to both self-supervised setting and supervised setting. In the supervised setting, baseline model is trained with labels in the end to end manner. In the self-supervised setting. Baseline model follow the linear evaluation scheme adopted by~\cite{velickovic2018deep, zhu2020deep}, which first trains models in an unsupervised fashion and then output the node representations to be evaluated by a logistic regression classifier with labels. Note that HyperGCL is in the self-supervised setting.  For Euclidean graph embeddings, we compare HyperGCL with supervised model, \ie, GCN~\cite{kipf2016semi}, GAT~\cite{velivckovic2017graph}, SGAE~\cite{hamilton2017inductive} and compare the Hyper with self supervised model \ie, GRACE~\cite{zhu2020deep}, COSTA~\cite{zhang2022costa}. For hyperbolic graph embeddings, we compare HyperGCL with supervised model, \ie, HGCN~\cite{chami2019hyperbolic}, HGAT~\cite{zhang2021hyperbolic}, HGNN~\cite{liu2019hyperbolic}. These methods impose the hyperbolic constrain via designing message passing in the hyperbolic space. We also compare HyperGCL with self supervised hyerbolic model \ie, HGCL~\cite{liu2022enhancing}.
Below we investigate both Euclidean embeddings and hyperbolic embeddings. Our model is compared in (i) self-supervised and (ii) supervised (model trained  with labeled data) setting. % In the supervised setting, the baseline model is trained end-to-end using labeled data. 
In the self-supervised setting, we follow the linear evaluation from Velivckovic \etal \cite{velivckovic2017graph} and Zhu \etal \cite{zhu2020deep}, where models are first trained in an unsupervised manner. Subsequently, the node representations are fed and evaluated with a logistic regression classifier on labeled data. Note that HyperGCL operates in a self-supervised setting. For Euclidean graph embeddings, we compare HyperGCL against several models. In the supervised setting, we consider GCN~\cite{kipf2016semi}, GAT~\cite{velivckovic2017graph}, SGAE~\cite{hamilton2017inductive}. In the self-supervised setting, we compare HyperGCL with typical graph contrastive learning method ~\cite{zhu2020deep, zhang2022costa}. For hyperbolic graph embeddings, we compare HyperGCL with  the supervised  hyperbolic models  HGCN~\cite{chami2019hyperbolic}, HGAT~\cite{zhang2021hyperbolic}, HGNN~\cite{liu2019hyperbolic}, HIE~\cite{yang2023hyperbolic}. %These models incorporate hyperbolic constraints by designing hyperbolic message passing operations.
We also compare HyperGCL with the self-supervised hyperbolic model HGCL~\cite{liu2022enhancing}.

% Below, we conduct several analyses on how the proposed component affects our model.

\vspace{0.1cm}
\noindent\textbf{Results.}  \tabref{tab:node_class_result}  compares HyperGCL to baseline methods.  HyperGCL achieves the highest performance in both the supervised GNN  and self-supervised GCL  settings. Notably, the performance gains of HyperGCL over Euclidean GNN models show  that graph contrastive learning can benefit from the hyperbolic geometry. Moreover, HyperGCL outperforms HGCL highlighting the benefit of hyperbolic alignment and outer shell isotropy losses. 

\vspace{-0.1cm}
\subsection{Result of Representation Learning in Collaborative Filtering}
In large-scale recommender systems, %the networks representing 
user-item relationships often exhibit scale-free or exponential expansion characteristics, making them particularly suitable for the hyperbolic embedding with the added the outer shell isotropy loss, \eg, by adapting HRCF~\cite{yang2022hrcf}. 

\vspace{0.1cm}
\noindent\textbf{Datasets.} We use three publicly available datasets Amazon-Book, Amazon-CD, and Yelp2020, which are also employed by the HRCF. 
% The statistics are summarized in ~\tabref{tab:cl_datasets} in the Appendix.

\vspace{0.1cm}
\noindent\textbf{Baselines.} To verify the effectiveness of our  method, the compared methods include both well-known or leading hyperbolic models and Euclidean baselines. For hyperbolic models, the HGCF~\cite{sun2021hgcf}, HVAE and HAE~\cite{liang2018variational} and are compared. HAE (HVAE) combines a (variational) autoencoder with hyperbolic geometry. We also include strong Euclidean baselines, \ie, LGCN~\cite{he2020lightgcn} and NGCF~\cite{wang2019neural}. 

\vspace{0.1cm}
\noindent\textbf{Setting.} To show that some equivalent of uniformity in the hyperbolic space is crucial for learning the hierarchical representation, we combine the proposed outer shell isotropy metric with the existing SOTA (\ie, HRCF~\cite{yang2022hrcf}) by adding $\mathcal{L}^{\mathbb{D}_c^d}_{U}$ as an auxiliary loss. We test the  model using the relevancy-based metric Recall@20 and the ranking-aware metric NDCG@20. 
% See Appendix~\ref{sec:detail_hyper_cl} for detailed setting.
In order to maintain a fair comparison and reduce the workload of our experiments, we closely adhere to the settings of HRCF~\cite{yang2022hrcf}. Specifically, we set the embedding size to 50 and fix the total training epochs at 500. The range of $\lambda$ values in the loss function is \{10, 15, 20, 25, 30\}, while the aggregation order is searched in range from 2 to 10. When it comes to the margin, we explore values within the range of \{0.1, 0.15, 0.2\}. To train the network parameters, we employ Riemannian SGD~\cite{bonnabel2013stochastic} with weight decay, using values from the range \{1e-4, 5e-4, 1e-3\}, along with learning rates of \{0.001, 0.0015, 0.002\}. RSGD is a technique that emulates stochastic gradient descent optimization while accounting for the geometry of the hyperbolic manifold~\cite{bonnabel2013stochastic}. The baseline settings of HAE, HAVE and HGCF are described in ~\cite{sun2021hgcf}.

\begin{table*}[t]
\vspace{-0.5cm}
\centering
\setlength\tabcolsep{2pt}
\fontsize{6}{6}\selectfont
\caption{Ablations on HyperGCL. Erank of embedding is measured in the ambient space (and \textcolor{blue}{(Erank)} in the \textcolor{blue}{tangent space}) of the encoder output. Eranks in  the ambient and tangent spaces correlate  well.}\label{tab:ab_study}
\begin{tabular}{ccccccccccccc}
\toprule[2pt]
\multirow{2}{*}{\textbf{Manifold}}& \multirow{2}{*}{\textbf{Align}} & \multirow{2}{*}{\textbf{Uni}} & \multicolumn{2}{c}{\textbf{PubMed}} &  \multicolumn{2}{c}{\textbf{CiteSeer}} &  \multicolumn{2}{c}{\textbf{Cora}}&  \multicolumn{2}{c}{\textbf{Disease}}&  \multicolumn{2}{c}{\textbf{Airport}}\\ 
\cmidrule{4-5}\cmidrule{6-7}\cmidrule{8-13}
                         &                                     &                                      &  \textbf{Acc.}               & \textbf{Erank}          &  \textbf{Acc.}              & \textbf{Erank}              & \textbf{Acc.}              &\textbf{Erank}& \textbf{Acc.}              &\textbf{Erank}& \textbf{Acc.}              &\textbf{Erank}  \\\midrule 
Euclidean                &$\mathcal{L}^{\mathbb{R}^d}_{A}$     & $\mathcal{L}^{\mathbb{R}^d}_{U}$       &  \ums{83.14}{0.18} &  5.22\textcolor{blue}{(5.20)}              & \ums{71.43}{0.52} & 23.01\textcolor{blue}{(23.03)}              & \ums{82.37}{0.27} &  5.50\textcolor{blue}{(5.47)} &  \ums{73.40}{0.24}& 2.19 \textcolor{blue}{(2.15)} &\ums{81.30}{0.21}& 2.34\textcolor{blue}{(2.32)}\\
Tangent                  &$\mathcal{L}^{\mathbb{R}^d}_{A}$     & $\mathcal{L}^{\mathbb{R}^d}_{U}$       &  \ms{82.34}{0.35}  &  4.79\textcolor{blue}{(4.78)}              &  \ms{71.42}{0.67} & 22.87\textcolor{blue}{(23.02)}              & \ms{81.34}{0.33} &  4.94\textcolor{blue}{(4.93)}  &  \ms{69.42}{0.45}& 2.08\textcolor{blue}{(2.05)} &\ms{79.53}{0.41} & 2.02\textcolor{blue}{(2.01)}\\ \midrule 
% Ambient                  &$\mathcal{L}^{\mathbb{R}^d}_{A}$     & $\mathcal{L}^{\mathbb{R}^d}_{U}$       &  \ums{80.14}{0.18} &                & \ums{71.43}{0.54} & 23.01              &\ums{82.37}{0.27} &  4.50  \\\midrule 
Hyperbolic               &$\mathcal{L}^{\mathbb{D}_c^d}_{A}$   & $\times$                               &  \ms{71.02}{0.13}  &   1.22\textcolor{blue}{(1.19)}             &  \ms{63.84}{0.43} & 5.93\textcolor{blue}{(5.92)}               & \ms{72.27}{0.12} &  1.23\textcolor{blue}{(1.22)} &  \ms{42.40}{0.64}& 1.02\textcolor{blue}{(1.01)} &\ms{52.31}{0.31}& 1.06\textcolor{blue}{(1.05)}\\
% Hyperbolic               & $\times$                            &$\mathcal{L}^{\mathbb{D}_c^d}_{U}$      &  \ms{71.02}{0.13}  &                &  \ms{63.84}{0.44} &                    & \ms{70.27}{0.12} &       \\
Hyperbolic               &$\mathcal{L}^{\mathbb{D}_c^d}_{A}$   & $\mathcal{L}^{\mathbb{R}^d}_{U}$     &  \ms{81.51}{0.37}  &  4.43\textcolor{blue}{(4.42)}              &  \ms{70.23}{0.41} & 21.01\textcolor{blue}{(21.01)}              & \ms{81.13}{0.44} &  4.50\textcolor{blue}{(4.47)} &  \ums{90.30}{0.82}& 3.59\textcolor{blue}{(3.58)} &\ums{91.70}{0.71}& 4.15\textcolor{blue}{(4.14)}\\ \midrule
 Hyperbolic  
                         &$\mathcal{L}^{\mathbb{D}_c^d}_{A}$   & $\mathcal{L}^{\mathbb{D}_c^d}_{U}$        &\ubms{85.14}{0.38}  &    6.89\textcolor{blue}{(6.88)}            &\ubms{73.43}{0.66} &24.75\textcolor{blue}{(24.72)}               & \ubms{84.47}{0.16}& 7.76\textcolor{blue}{(7.75)} & \ubms{94.50}{0.43} & 4.79\textcolor{blue}{(4.78)} &\ubms{93.55}{1.11}& 5.25\textcolor{blue}{(5.23)}\\
\bottomrule[2pt]
\end{tabular}
\end{table*}

\vspace{0.1cm}
\noindent\textbf{Results.} \tabref{tab:hyper_cf} shows that HyperGCL  outperforms the baseline models on all datasets across both Recall@20 and NDCG@20. %, demonstrating the effectiveness of our proposal. 
Hyperbolic models equipped with the ranking loss (\ie, HGCF and HRCF) show a significant advantage over their Euclidean counterparts (\ie, LGCN), indicating the superiority of hyperbolic geometry for modeling user-item networks.
Using $\mathcal{L}^{\mathbb{D}_c^d}_{U}$ to adjust the embedding space   improves results by alleviating the dimensional collapse. % can help avoid the popularity bias in the recommendation task.

\vspace{-0.2cm}
\subsection{Analysis}

\noindent\textbf{Ablation Study.} 
% Below we analyze various components of HyperGCL in
 \tabref{tab:ab_study}  shows that uniformity loss (\eqref{eq:euc_align_uni}) in either the tangent or hyperbolic space do not achieve higher effective rank than our method, and yield low results. %We find that 
  There is nothing to prevent the network from learning a constant embedding for all nodes if only the alignment loss (the third row) is used. Applying augmentation prevents the full collapse but its ERank is low. Our HyperGCL (last row) yields a higher ERank and the best results. % among all variants.

\begin{table}
%\vspace{-20px}
\setlength\tabcolsep{3pt}
\fontsize{8}{6.5}\selectfont
% \vspace{-1cm}
% \caption{Ablatio.~\label{tab:mean_var}}
\caption{Test performance \wrt the isotropic Normal distribution with different mean centers.~\label{tab:mean_gauss}}
    \centering
    \begin{tabular}{lcccccc}
\toprule[2pt]
\multirow{2}{*}{Gaussian}                             &\multicolumn{2}{c}{\textbf{PubMed}}    & \multicolumn{2}{c}{\textbf{CiteSeer}}   & \multicolumn{2}{c}{\textbf{Cora}}     \\
\cmidrule{2-7}
                                     & Acc.          &Erank                    & Acc.          &Erank                  & Acc.          &Erank \\ \midrule
$\mathcal{N}(0.0\cdot\mathbf{1},\mathbf{I}) $       &\ubms{85.56}{0.21} &\textcolor{blue}{\textbf{6.89}}   &\ubms{74.43}{0.35}&24.75   &\ubms{84.77}{0.46} &\textcolor{blue}{\textbf{7.76}} \\ 
$\mathcal{N}(0.5\cdot\mathbf{1},\mathbf{I})$                           &\ums{80.13}{0.24}  &\textcolor{blue}{4.27}       &\ms{69.43}{0.22}  &\textcolor{blue}{18.34}        &\ums{80.27}{0.11}   &\textcolor{blue}{4.31} \\
$\mathcal{N}(1.0\cdot\mathbf{1},\mathbf{I})$                           &\ms{80.79}{0.42}   &\textcolor{blue}{4.20}       &\ums{70.25}{0.15} &\textcolor{blue}{20.65}        &\ms{79.47}{0.35}    &\textcolor{blue}{4.45} \\\midrule
$\mathcal{N}(\mathbf{0},\mathbf{I}_{0.3}) $                            &\ms{84.24}{0.18}   &\textcolor{blue}{5.59}       &\ms{71.82}{0.47}  &\textcolor{blue}{20.41}        &\ms{82.32}{0.13}    &\textcolor{blue}{4.98} \\
$\mathcal{N}(\mathbf{0},\mathbf{I}_{0.5}) $                            &\ms{82.34}{0.23}   &\textcolor{blue}{5.02}       &\ms{70.13}{0.34}  &\textcolor{blue}{16.32}        &\ms{80.17}{0.15}    &\textcolor{blue}{4.20} \\
$\mathcal{N}(\mathbf{0},\mathbf{I}_{0.7}) $                            &\ms{78.24}{0.64}   &\textcolor{blue}{4.15}       &\ms{67.43}{0.35}  &\textcolor{blue}{9.69 }        &\ms{75.25}{0.14}    &\textcolor{blue}{2.85} \\
\bottomrule[2pt]
\end{tabular}
%\vspace{-10px}
\end{table}

% We observe that both alignment and uniformity lead to performance gains compared to using either of them alone. Furthermore, optimizing the $\mathcal{L}_{a}$ alone in either the hyperbolic or Euclidean space leads to a significant performance drop due to %. This is because minimizing the distance between two embeddings can easily cause 
% the dimensional collapse.  We also observe that optimizing $\mathcal{L}_{u}$ alone can also influence the performance, but the performance drop is minor. We also note the ablation shows that the hyperbolic graph embedding is better than the Euclidean embedding for graph benchmarks. 

 % non-zero centered Normal distribution (tangent space)
 
\vspace{0.1cm}
\noindent\textbf{The Impact of Gaussian Centering and Non-isotropy (tangent space).}
To  explore  the impact of the above two factors, % 
%on the outer shell isotropy loss.  
we  scale vector $\mathbf{1}$ by some constant
  factor (see Table \ref{tab:mean_gauss}). To show non-isotropy  has a negative impact on results, we randomly sample a fraction of $p={0.3, 0.5, 0.7}$ from the diagonal elements of  $\mathbf{I}$ and set them to be $0.01$ to simulate the non-isotropic Normal distribution (covariance matrix denoted by $\mathbf{I}_{p}$). When moving from the zero-mean or violating isotropy of Gaussian, the mode collapse in the hyperbolic space occurs, as in Figure \ref{fig:gaussain_dis_a}. Table \ref{tab:mean_gauss} confirms the detrimental impact of non-isotropy and non-zero centering on experimental results.

%\vspace{-0.2cm}
\begin{tcolorbox}[width=1\linewidth, colframe=blackish, colback=beaublue, boxsep=0mm, arc=2mm, left=2mm, right=2mm, top=1mm, bottom=1mm]
%\vspace{0.1cm}
\noindent\textbf{Discussion on ERank}. Table \ref{tab:ab_study} reveals that Effective Ranks in the ambient space and the tangent space are highly correlated (as expected as these spaces are connected via non-linear mapping).  Theorem~\ref{theorem:convex} and Theorem~\ref{theorem:opt} show that imposing the isotropic Gaussian in the tangent space on features improves its ERank. Ergo, we improve the ERank in the ambient space by promoting the outer shell isotropy in the ambient space, which means preventing the leaf and height collapse.
\end{tcolorbox}

\vspace{0.1cm}
\noindent
\textbf{Impact of Curvature $c$.} 
As the curvature parameter $c$ controls the depth of hierarchy (height of the tree embedding)\footnote{As features are not allowed to reach the Poincaré boundary due to $\epsilon\!>\!0$, we can control the depth of the tree embedding.}, we analyze its effect on results. The notion of height-level uniformity is related to the value of $c$. The larger $c$ is, the more concentration of the distribution towards the tree root. \figref{fig:c_effect} shows that HyperGCL achieves the best result for different $c$  meaning that the height-level uniformity is data-dependent and related to sparsity of the datasets (sparsity is indicated in caption brackets of \figref{fig:c_effect}), \eg, graphs with relatively larger density require smaller $c$.

\begin{figure}[t]
% \vspace{-10px}
    \centering
\begin{subfigure}{0.25\textwidth}
        \centering
        \includegraphics[width=\textwidth]{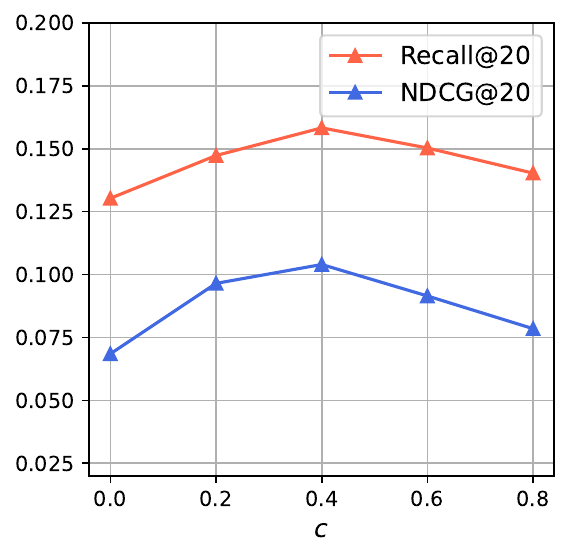}
         %\vspace{-9px}
        \caption{Amazon-CD($0.01$)}
        \label{fig:c_cd}
\end{subfigure}
 \hspace{5pt}
%\hfill
\begin{subfigure}{0.25\textwidth}
        \centering
    \includegraphics[width=\textwidth]{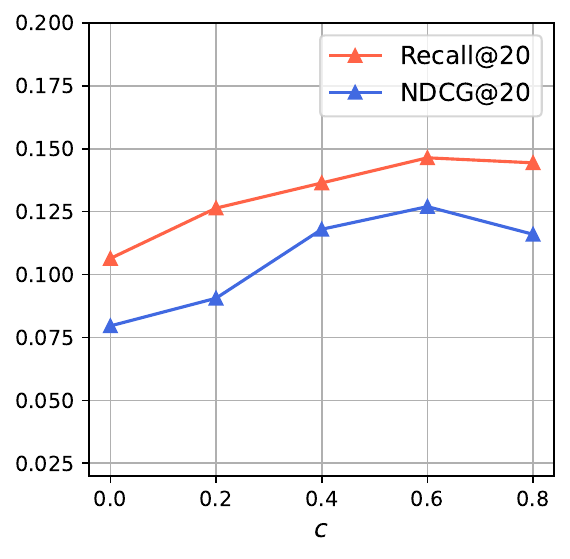}
     %\vspace{-9px}
        \caption{Amazon-BK($0.008$)}
        \label{fig:c_book}
\end{subfigure}
 \hspace{5pt}
%\hfill
\begin{subfigure}{0.25\textwidth}
        \centering
    \includegraphics[width=\textwidth]{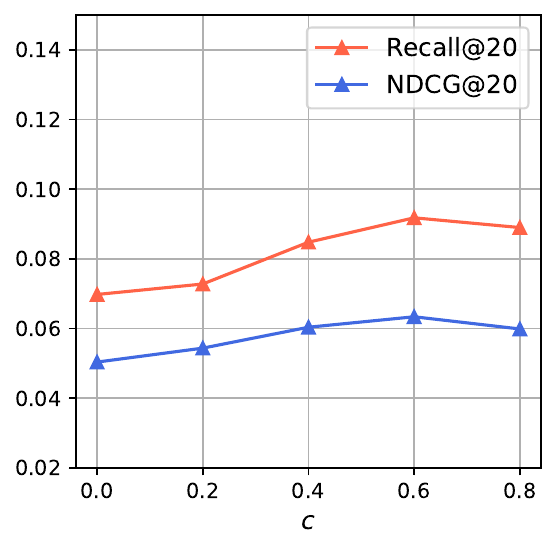}
     %\vspace{-9px}
        \caption{Yelp2020($0.004$)}
        \label{fig:c_yelp}
\end{subfigure}
%\vspace{-10px}
\caption{Performance \wrt the value of curvature $c$. In caption brackets, we indicate the dataset sparsity.\label{fig:c_effect}}
    % \vspace{-10pt}
%\vspace{-15px}
\end{figure}

\section{Related Works}

\noindent\textbf{Hyperbolic Graph Neural Networks.}
%Hyperbolic geometry has attracted a lot of attention in network science communities.
Recent works~\cite{liu2019HGNN,chami2019hyperbolic,zhang2021hyperbolic,he2020lightgcn,liu2022enhancing} extend GNNs to the hyperbolic space. HGNN~\cite{liu2019HGNN}, HGCN~\cite{chami2019hyperbolic}, \& HGAT~\cite{zhang2021hyperbolic} use graph convolutions (GC) in the tangent space.  LGCN~\cite{he2020lightgcn} uses GC on the hyperbolic manifold. HGNN~\cite{liu2019HGNN} %(graph classification)  
proposes dynamic graph embeddings. HGAT~\cite{zhang2021hyperbolic} %(node classification and clustering) 
introduces a hyperbolic attention-based GC. 
% via in M\"{o}bius gyrovector spaces 
 HGCN~\cite{chami2019hyperbolic} %introduces a local aggregation scheme and 
develops a learnable curvature model. % for hyperbolic learning. 
LGCN~\cite{he2020lightgcn} aggregates the neighborhood  via centroid of Lorentzian distance. HGCL~\cite{liu2022enhancing}  uses CL to improve  GCL. 
% Besides, works in~\cite{gu2019learning,zhu2020gil} propose to learn representations over mixed spaces. 

\vspace{0.1cm}
\noindent\textbf{Graph Contrastive Learning.} 
Inspired by CL methods in vision and NLP~\cite{chen2020simple, gao2021simcse}, CL has also been adapted to the graph domain. By adapting DeepInfoMax~\cite{bachman2019learning} to graph representation learning, DGI~\cite{velickovic2019deep} learns  embedding 
by maximizing the mutual information to discriminate between nodes of original and corrupted graphs. REFINE \cite{zhu2021refine} uses a simple negative sampling term inspired by skip-gram models. Fisher-Bures Adv. GCN \cite{uai_ke}  perturbs  graph Laplacian. Inspired by SimCLR~\cite{chen2020simple}, GRACE~\cite{zhu2020deep} correlates graph views for node-level task % by pushing closer representations of the same node in different views and pushing apart representations of different nodes. 
while GraphCL \cite{hafidi2020graphcl} learns embeddings for graph-level tasks. COLES \cite{zhu2021contrastive}, EASE \cite{Zhu_2022_CVPR} and GLEN \cite{glen} introduce the negative sampling into Laplacian Eigenmaps. Alternatives to CL are based on image/feature masking \cite{hondru2025maskedimagemodelingsurvey,ropim}.

\section{Conclusions}

We believe ours is the first work investigating the notion of dimensional collapse for hyperbolic graph embedding. 
We have shown that i) a poor quality hyperbolic embedding space results in the leaf- and height-level collapse of tree-equivalent cases, and ii) imposing a zero-centered isotropic Normal distribution in the tangent plane at $\boldsymbol{0}$, $\mathcal{T}_{\boldsymbol{0}} \mathbb{D}^d_c$ leads to  the leaf- and height-level uniformity in the hyperbolic space. Such a notion of uniformity translates into the outer shell isotropy in the ambient space of the hyperbolic manifold. Such a zero-centered isotropic Normal distribution can be mapped to the Poincaré disk via the exponential map \& the KL-divergence.

\textbf{Acknowledgments.} The work described in this paper was partially supported by the Research Grants Council of the Hong Kong Special Administrative Region, China (CUHK 2410072, RGC R1015-23). Hao Zhu and Piotr Koniusz are supported by the CSIRO’s Science Digital and Advanced Engineering
Biology Future Science Platforms.

{\small
\bibliographystyle{ieee_fullname}
\bibliography{acmart.bib}

\begin{thebibliography}{10}\itemsep=-1pt

\bibitem{adcock2013tree}
Aaron~B Adcock, Blair~D Sullivan, and Michael~W Mahoney.
\newblock Tree-like structure in large social and information networks.
\newblock In {\em 2013 IEEE 13th International Conference on Data Mining},
  pages 1--10. IEEE, 2013.

\bibitem{bachman2019learning}
Philip Bachman, R~Devon Hjelm, and William Buchwalter.
\newblock Learning representations by maximizing mutual information across
  views.
\newblock {\em arXiv preprint arXiv:1906.00910}, 2019.

\bibitem{bonnabel2013stochastic}
Silvere Bonnabel.
\newblock Stochastic gradient descent on riemannian manifolds.
\newblock {\em IEEE Transactions on Automatic Control}, 58(9):2217--2229, 2013.

\bibitem{bronstein2017geometric}
Michael~M Bronstein, Joan Bruna, Yann LeCun, Arthur Szlam, and Pierre
  Vandergheynst.
\newblock Geometric deep learning: going beyond euclidean data.
\newblock {\em IEEE Signal Processing Magazine}, 34(4):18--42, 2017.

\bibitem{cannon1997hyperbolic}
James~W Cannon, William~J Floyd, Richard Kenyon, Walter~R Parry, et~al.
\newblock Hyperbolic geometry.
\newblock {\em Flavors of geometry}, 31(59-115):2, 1997.

\bibitem{chami2019hyperbolic}
Ines Chami, Zhitao Ying, Christopher R{\'e}, and Jure Leskovec.
\newblock Hyperbolic graph convolutional neural networks.
\newblock {\em Advances in neural information processing systems}, 32, 2019.

\bibitem{chen2020simple}
Ting Chen, Simon Kornblith, Mohammad Norouzi, and Geoffrey Hinton.
\newblock A simple framework for contrastive learning of visual
  representations.
\newblock In {\em International conference on machine learning}, pages
  1597--1607. PMLR, 2020.

\bibitem{chen2020improved}
Xinlei Chen, Haoqi Fan, Ross Girshick, and Kaiming He.
\newblock Improved baselines with momentum contrastive learning.
\newblock {\em arXiv preprint arXiv:2003.04297}, 2020.

\bibitem{cherian2021learning}
Anoop Cherian, Panagiotis Stanitsas, Jue Wang, Mehrtash Harandi, Vassilios
  Morellas, and Nikolaos Papanikolopoulos.
\newblock Learning log-determinant divergences for positive definite matrices.
\newblock {\em IEEE Transactions on Pattern Analysis and Machine Intelligence},
  44(9):5088--5102, 2021.

\bibitem{gap_dc}
Junhao Dong, Piotr Koniusz, Xinghua Qu, and Yew-Soon Ong.
\newblock Stabilizing modality gap \& lowering gradient norms improve zero-shot
  adversarial robustness of vlms.
\newblock {\em 31st SIGKDD Conference on Knowledge Discovery and Data Mining},
  2025.

\bibitem{fang2023poincare}
Pengfei Fang, Mehrtash Harandi, Zhenzhong Lan, and Lars Petersson.
\newblock Poincar{\'e} kernels for hyperbolic representations.
\newblock {\em International Journal of Computer Vision}, 131(11):2770--2792,
  2023.

\bibitem{ganea2018hyperbolic}
Octavian Ganea, Gary B{\'e}cigneul, and Thomas Hofmann.
\newblock Hyperbolic entailment cones for learning hierarchical embeddings.
\newblock In {\em International Conference on Machine Learning}, pages
  1646--1655. PMLR, 2018.

\bibitem{gao2021simcse}
Tianyu Gao, Xingcheng Yao, and Danqi Chen.
\newblock Simcse: Simple contrastive learning of sentence embeddings.
\newblock {\em arXiv preprint arXiv:2104.08821}, 2021.

\bibitem{golub2013matrix}
Gene~H Golub and Charles~F Van~Loan.
\newblock {\em Matrix computations}.
\newblock JHU press, 2013.

\bibitem{grill2020bootstrap}
Jean-Bastien Grill, Florian Strub, Florent Altch{\'e}, Corentin Tallec, Pierre
  Richemond, Elena Buchatskaya, Carl Doersch, Bernardo Avila~Pires, Zhaohan
  Guo, Mohammad Gheshlaghi~Azar, et~al.
\newblock Bootstrap your own latent-a new approach to self-supervised learning.
\newblock {\em Advances in neural information processing systems},
  33:21271--21284, 2020.

\bibitem{hafidi2020graphcl}
Hakim Hafidi, Mounir Ghogho, Philippe Ciblat, and Ananthram Swami.
\newblock Graphcl: Contrastive self-supervised learning of graph
  representations.
\newblock {\em arXiv preprint arXiv:2007.08025}, 2020.

\bibitem{ropim}
Maryam Haghighat, Peyman Moghadam, Shaheer Mohamed, and Piotr Koniusz.
\newblock Pre-training with random orthogonal projection image modeling.
\newblock In {\em The Twelfth International Conference on Learning
  Representations}, 2024.

\bibitem{hamilton2017inductive}
Will Hamilton, Zhitao Ying, and Jure Leskovec.
\newblock Inductive representation learning on large graphs.
\newblock {\em Advances in neural information processing systems}, 30, 2017.

\bibitem{he2020lightgcn}
Xiangnan He, Kuan Deng, Xiang Wang, Yan Li, Yongdong Zhang, and Meng Wang.
\newblock Lightgcn: Simplifying and powering graph convolution network for
  recommendation.
\newblock In {\em Proceedings of the 43rd International ACM SIGIR conference on
  research and development in Information Retrieval}, pages 639--648, 2020.

\bibitem{hondru2025maskedimagemodelingsurvey}
Vlad Hondru, Florinel~Alin Croitoru, Shervin Minaee, Radu~Tudor Ionescu, and
  Nicu Sebe.
\newblock Masked image modeling: A survey.
\newblock {\em arXiv: 2408.06687}, 2025.

\bibitem{jing2021understanding}
Li Jing, Pascal Vincent, Yann LeCun, and Yuandong Tian.
\newblock Understanding dimensional collapse in contrastive self-supervised
  learning.
\newblock {\em arXiv preprint arXiv:2110.09348}, 2021.

\bibitem{kipf2016semi}
Thomas~N Kipf and Max Welling.
\newblock Semi-supervised classification with graph convolutional networks.
\newblock {\em arXiv preprint arXiv:1609.02907}, 2016.

\bibitem{krioukov2010hyperbolic}
Dmitri Krioukov, Fragkiskos Papadopoulos, Maksim Kitsak, Amin Vahdat, and
  Mari{\'a}n Bogun{\'a}.
\newblock Hyperbolic geometry of complex networks.
\newblock {\em Physical Review E}, 82(3):036106, 2010.

\bibitem{liang2018variational}
Dawen Liang, Rahul~G Krishnan, Matthew~D Hoffman, and Tony Jebara.
\newblock Variational autoencoders for collaborative filtering.
\newblock In {\em Proceedings of the 2018 world wide web conference}, pages
  689--698, 2018.

\bibitem{liu2022enhancing}
Jiahong Liu, Menglin Yang, Min Zhou, Shanshan Feng, and Philippe
  Fournier-Viger.
\newblock Enhancing hyperbolic graph embeddings via contrastive learning.
\newblock {\em arXiv preprint arXiv:2201.08554}, 2022.

\bibitem{liu2019spectral}
Kanglin Liu, Wenming Tang, Fei Zhou, and Guoping Qiu.
\newblock Spectral regularization for combating mode collapse in gans.
\newblock In {\em Proceedings of the IEEE/CVF international conference on
  computer vision}, pages 6382--6390, 2019.

\bibitem{liu2019hyperbolic}
Qi Liu, Maximilian Nickel, and Douwe Kiela.
\newblock Hyperbolic graph neural networks.
\newblock {\em Advances in neural information processing systems}, 32, 2019.

\bibitem{liu2019HGNN}
Qi Liu, Maximilian Nickel, and Douwe Kiela.
\newblock Hyperbolic graph neural networks.
\newblock In {\em NeurIPS}, pages 8230--8241, 2019.

\bibitem{NEURIPS2019_0ec04cb3}
Emile Mathieu, Charline Le~Lan, Chris~J. Maddison, Ryota Tomioka, and Yee~Whye
  Teh.
\newblock Continuous hierarchical representations with poincar\'{e} variational
  auto-encoders.
\newblock In H. Wallach, H. Larochelle, A. Beygelzimer, F. d\textquotesingle
  Alch\'{e}-Buc, E. Fox, and R. Garnett, editors, {\em Advances in Neural
  Information Processing Systems}, volume~32. Curran Associates, Inc., 2019.

\bibitem{nickel2017poincare}
Maximillian Nickel and Douwe Kiela.
\newblock Poincaré embeddings for learning hierarchical representations.
\newblock {\em Advances in neural information processing systems}, 30, 2017.

\bibitem{paeng2011brownian}
Seong-Hun Paeng.
\newblock Brownian motion on manifolds with time-dependent metrics and
  stochastic completeness.
\newblock {\em Journal of Geometry and Physics}, 61(5), 2011.

\bibitem{7098875}
Olivier Roy and Martin Vetterli.
\newblock The effective rank: A measure of effective dimensionality.
\newblock In {\em 2007 15th European Signal Processing Conference}, pages
  606--610, 2007.

\bibitem{sarkar2011low}
Rik Sarkar.
\newblock Low distortion delaunay embedding of trees in hyperbolic plane.
\newblock In {\em International symposium on graph drawing}, pages 355--366.
  Springer, 2011.

\bibitem{sun2021hgcf}
Jianing Sun, Zhaoyue Cheng, Saba Zuberi, Felipe P{\'e}rez, and Maksims Volkovs.
\newblock Hgcf: Hyperbolic graph convolution networks for collaborative
  filtering.
\newblock In {\em Proceedings of the Web Conference 2021}, pages 593--601,
  2021.

\bibitem{uai_ke}
Ke Sun, Piotr Koniusz, and Zhen Wang.
\newblock Fisher-bures adversary graph convolutional networks.
\newblock {\em Conference on Uncertainty in Artificial Intelligence},
  115:465--475, 2019.

\bibitem{tifrea2018poincar}
Alexandru Tifrea, Gary B{\'e}cigneul, and Octavian-Eugen Ganea.
\newblock Poincar$\backslash$'e glove: Hyperbolic word embeddings.
\newblock {\em arXiv preprint arXiv:1810.06546}, 2018.

\bibitem{tschannen2019mutual}
Michael Tschannen, Josip Djolonga, Paul~K Rubenstein, Sylvain Gelly, and Mario
  Lucic.
\newblock On mutual information maximization for representation learning.
\newblock {\em arXiv preprint arXiv:1907.13625}, 2019.

\bibitem{velivckovic2017graph}
Petar Veli{\v{c}}kovi{\'c}, Guillem Cucurull, Arantxa Casanova, Adriana Romero,
  Pietro Lio, and Yoshua Bengio.
\newblock Graph attention networks.
\newblock {\em arXiv preprint arXiv:1710.10903}, 2017.

\bibitem{velickovic2019deep}
Petar Velickovic, William Fedus, William~L. Hamilton, Pietro Li{\`{o}}, Yoshua
  Bengio, and R.~Devon Hjelm.
\newblock Deep graph infomax.
\newblock In {\em ICLR}, 2019.

\bibitem{wang2020understanding}
Tongzhou Wang and Phillip Isola.
\newblock Understanding contrastive representation learning through alignment
  and uniformity on the hypersphere.
\newblock In {\em International Conference on Machine Learning}, pages
  9929--9939. PMLR, 2020.

\bibitem{wang2019neural}
Xiang Wang, Xiangnan He, Meng Wang, Fuli Feng, and Tat-Seng Chua.
\newblock Neural graph collaborative filtering.
\newblock In {\em Proceedings of the 42nd international ACM SIGIR conference on
  Research and development in Information Retrieval}, pages 165--174, 2019.

\bibitem{lagranges}
Wikipedia.org.
\newblock Lagrange multiplier, 2024.

\bibitem{wu2018unsupervised}
Zhirong Wu, Yuanjun Xiong, Stella~X Yu, and Dahua Lin.
\newblock Unsupervised feature learning via non-parametric instance
  discrimination.
\newblock In {\em Proceedings of the IEEE conference on computer vision and
  pattern recognition}, pages 3733--3742, 2018.

\bibitem{yang2022hrcf}
Menglin Yang, Min Zhou, Jiahong Liu, Defu Lian, and Irwin King.
\newblock Hrcf: Enhancing collaborative filtering via hyperbolic geometric
  regularization.
\newblock In {\em Proceedings of the ACM Web Conference 2022}, pages
  2462--2471, 2022.

\bibitem{yang2023hyperbolic}
Menglin Yang, Min Zhou, Rex Ying, Yankai Chen, and Irwin King.
\newblock Hyperbolic representation learning: Revisiting and advancing.
\newblock {\em arXiv preprint arXiv:2306.09118}, 2023.

\bibitem{zhang2018link}
Muhan Zhang and Yixin Chen.
\newblock Link prediction based on graph neural networks.
\newblock {\em Advances in neural information processing systems}, 31, 2018.

\bibitem{zhang2021hyperbolic}
Yiding Zhang, Xiao Wang, Chuan Shi, Xunqiang Jiang, and Yanfang Ye.
\newblock Hyperbolic graph attention network.
\newblock {\em IEEE Transactions on Big Data}, 8(6):1690--1701, 2021.

\bibitem{zhang2024geometric}
Yifei Zhang, Hao Zhu, Zixing Song, Yankai Chen, Xinyu Fu, Ziqiao Meng, Piotr
  Koniusz, and Irwin King.
\newblock Geometric view of soft decorrelation in self-supervised learning.
\newblock In {\em Proceedings of the 30th ACM SIGKDD Conference on Knowledge
  Discovery and Data Mining}, pages 4338--4349, 2024.

\bibitem{zhang2022costa}
Yifei Zhang, Hao Zhu, Zixing Song, Piotr Koniusz, and Irwin King.
\newblock {COSTA}: Covariance-preserving feature augmentation for graph
  contrastive learning.
\newblock In {\em Proceedings of the 28th ACM SIGKDD Conference on Knowledge
  Discovery and Data Mining}, pages 2524--2534, 2022.

\bibitem{zhang2023spectral}
Yifei Zhang, Hao Zhu, Zixing Song, Piotr Koniusz, and Irwin King.
\newblock Spectral feature augmentation for graph contrastive learning and
  beyond.
\newblock In {\em Proceedings of the AAAI Conference on Artificial
  Intelligence}, volume~37, pages 11289--11297, 2023.

\bibitem{zhang2023mitigating}
Yifei Zhang, Hao Zhu, Zixing Song, Piotr Koniusz, Irwin King, et~al.
\newblock Mitigating the popularity bias of graph collaborative filtering: A
  dimensional collapse perspective.
\newblock {\em Advances in Neural Information Processing Systems},
  36:67533--67550, 2023.

\bibitem{zhu2021refine}
Hao Zhu and Piotr Koniusz.
\newblock Refine: Random range finder for network embedding.
\newblock In {\em ACM Conference on Information and Knowledge Management},
  2021.

\bibitem{Zhu_2022_CVPR}
Hao Zhu and Piotr Koniusz.
\newblock {EASE}: Unsupervised discriminant subspace learning for transductive
  few-shot learning.
\newblock In {\em Proceedings of the IEEE/CVF Conference on Computer Vision and
  Pattern Recognition (CVPR)}, pages 9078--9088, June 2022.

\bibitem{glen}
Hao Zhu and Piotr Koniusz.
\newblock Generalized laplacian eigenmaps.
\newblock In {\em Advances in Neural Information Processing Systems},
  volume~35, pages 30783--30797. Curran Associates, Inc., 2022.

\bibitem{zhu2021contrastive}
Hao Zhu, Ke Sun, and Peter Koniusz.
\newblock Contrastive laplacian eigenmaps.
\newblock {\em Advances in Neural Information Processing Systems},
  34:5682--5695, 2021.

\bibitem{zhu2020deep}
Yanqiao Zhu, Yichen Xu, Feng Yu, Qiang Liu, Shu Wu, and Liang Wang.
\newblock Deep graph contrastive representation learning.
\newblock {\em arXiv:2006.04131}, 2020.

\end{thebibliography}
}

\appendix

\section{Supplementary Material (Appendices)}

\subsection{KL divergence for the multivatiate Normal distributions}
\label{app:kl-proof1}
The KL divergence between multivariate Normal distributions $p$ and $q$ represented by $\mathcal{N}(\boldsymbol{\mu},\boldsymbol{\Sigma})$ and  $\mathcal{N}(\boldsymbol{0},\boldsymbol{I})$ is derived as:

\vspace{-0.3cm}
\begin{equation}
\setlength{\abovedisplayskip}{3pt}
\begin{aligned}
D_{KL}\left(p, q \right)&= \frac{1}{2}\left[\operatorname{tr}\left(\mathbf{I}^{-1} \boldsymbol{\Sigma}\right)+\left(\boldsymbol{0}-\boldsymbol{\mu}\right)^{T} \mathbf{I}^{-1}\left(\boldsymbol{0}-\boldsymbol{\mu}\right)-d+\log \frac{\operatorname{det} \mathbf{I}}{\operatorname{det}\boldsymbol{\Sigma}}\right]\\[-6px]
=\qquad\! & \\[-8px]
       D(\mathbf{\Sigma},\boldsymbol{\mu}) & = \frac{1}{2}\left[\operatorname{tr}(\boldsymbol{\Sigma})-\log \operatorname{det}(\boldsymbol{\Sigma})-d + \|\boldsymbol{\mu}\|_2^2\right].
        \label{eq:tanh_kl)}\\[-16px]
\end{aligned}
\setlength{\belowdisplayskip}{3pt}
\vspace{0.2cm}
\end{equation}
\subsection{KL Divergence between the multivariate Laplace distributions with $(\mathbf{0},\mathbf{\Sigma})$ and $(\mathbf{0},\mathbf{I})$ parameters}
\label{app:lap}

Recall two Laplace distributions  $\mathcal{L}_{ap}(\mathbf{0},\mathbf{\Sigma})$ and $\mathcal{L}_{ap}(\mathbf{0},\mathbf{I})$ from \eqref{eq:lap1} and \eqref{eq:lap2}:
\begin{align}
&p(\mathbf{y})\!=\!\frac{2}{(2 \pi)^{d/2}\operatorname{det}(\boldsymbol{\Sigma})^{1/2}}\left(\frac{\boldsymbol{y}^T \boldsymbol{\Sigma}^{-1} \mathbf{y}}{2}\right)^{v / 2} \!\!K_v\left(\sqrt{2 \boldsymbol{y}^T \boldsymbol{\Sigma}^{-1} \boldsymbol{y}}\right),\label{eq:lap5}\\
&q(\mathbf{y})\!=\!\frac{2}{(2 \pi)^{d/2}}\left(\frac{\lVert\mathbf{y}\rVert_2^2}{2}\right)^{v / 2} \!\!K_v\left(\sqrt{2}\lVert\mathbf{y}\rVert_2\right)\!=\!\prod\limits_{i=1}^d\frac{1}{2}\exp\big(-\lvert y_i\rvert\big).\label{eq:lap6}
\end{align}

As \eqref{eq:lap6} is isotropic, the value of KL divergence between $p(\mathbf{y})$ and $q(\mathbf{y})$ does not depend on the rotation matrix $\mathbf{U}$ from the SVD decomposition $\mathbf{\Sigma}\!=\!\mathbf{U}\boldsymbol{\Lambda}\mathbf{U}^T$. Therefore, we can readily say that $D_{KL}(p,q)\!=\!D_{KL}(p',q)$ where $p'(\mathbf{y})$ is:
\begin{align}
&p'(\mathbf{y})\!=\!\prod\limits_{i=1}^d\frac{1}{2\sqrt{\lambda_i}}\exp\left(-\Big\lvert \frac{y_i}{\sqrt{\lambda_i}}\Big\rvert\right).\label{eq:lap7}
\end{align}
%\mathcal{L}_{ap}(\mathbf{0},\mabtbf{\Sigma})

Then we have:
\begin{equation}
   \begin{aligned}
& D_{KL}(p, q)=\int _{-\infty }^{\infty }p'(\mathbf{y})\ \log \left({\frac {p'(\mathbf{y})}{q(\mathbf{y})}}\right)\ \mathrm {d} \ \!\mathbf{y}\\
&=\operatorname{\mathbb{E}}\limits_{\mathbf{y}\sim\mathcal{L}_{ap}(\mathbf{0},\mathbf{\Sigma})}\Big[\log\big(p'(\mathbf{y})\big)-\log\big(q(\mathbf{y})\big)\Big]\\
&=\operatorname{\mathbb{E}}\limits_{\!\!\!\!\!\!\!\!\!\!\!\!\!\!\!\!(y_1,\ldots,y_d)\sim\left\{\substack{\mathcal{L}_{ap}(0,\sqrt{\lambda_1})\\\times\cdots\times\\\mathcal{L}_{ap}(0,\sqrt{\lambda_d})}\right.}\!\sum\limits_{i=1}^d
\underbrace{
\left[
\log\left(\frac{1}{2\sqrt{\lambda_i}}\right)-\frac{\lvert y_i\rvert}{\sqrt{\lambda_i}}-\log{\frac{1}{2}}+\lvert y_i\rvert
\right]
}_{-\frac{1}{2}\log(\lambda_i)-\frac{\lvert y_i\rvert}{\sqrt{\lambda_i}}+\lvert y_i\rvert }\\
&=-\frac{1}{2}\log\operatorname{det}(\boldsymbol{\Lambda})\!+\!\!\sum\limits_{i=1}^d\!
\left[
\operatorname{\mathbb{E}}\limits_{y_i\sim\mathcal{L}_{ap}(0,\sqrt{\lambda_i})}\lvert y_i\rvert\!-\!\operatorname{\mathbb{E}}\limits_{y_i\sim\mathcal{L}_{ap}(0,1)}\lvert y_i\rvert
\right]\\
&=-\frac{1}{2}\log\operatorname{det}(\boldsymbol{\Lambda})\!+\!\!\sum\limits_{i=1}^d\!
\Big[
\operatorname{\mathbb{E}}\limits_{z_i\sim\mathcal{E}_{xp}\left(0,\lambda_i^{-1/2}\right)} z_i-\operatorname{\mathbb{E}}\limits_{z_i\sim\mathcal{E}_{xp}(0,1)} z_i
\Big]\\
&=-\frac{1}{2}\log\operatorname{det}(\boldsymbol{\Lambda})\!+\!\!\sum\limits_{i=1}^d\!
\left[
\sqrt{\lambda_i}\!-\!1
\right]\!=\!-\frac{1}{2}\log\operatorname{det}(\boldsymbol{\Sigma})\!+\!\operatorname{tr}\big(\boldsymbol{\Sigma}^{1/2}\big)\!-\!d.
\end{aligned} 
\end{equation}

\subsection{Proof of Theorem~\ref{th:rdm}}

\noindent\textbf{Radial component dependence.} Let $\mathbf{v} = r\mathbf{u}$ where $r = |\mathbf{v}|$ and $\mathbf{u}$ is a unit vector. The Jacobian of the transformation $f$ is:
\begin{equation}
J_K(r) = \begin{cases}
\frac{\sinh(\sqrt{|K|}r)}{\sqrt{|K|}r} & \text{for } K < 0 \\
1 & \text{for } K = 0 \\
\frac{\sin(\sqrt{K}r)}{\sqrt{K}r} & \text{for } K > 0
\end{cases}
\end{equation}
The distribution $p_Z(\mathbf{z})$ can be expressed using the change of variables formula:
$p_Z(\mathbf{z}) = p_V(f^{-1}(\mathbf{z})) \cdot |J_K(r)|^{-n}$
where $n$ is the dimension of the manifold.
The radial component of $p_Z(\mathbf{z})$, denoted as $p_R(s)$ where $s = |\mathbf{z}|$, is:
$p_R(s) = C_n \cdot \phi(f^{-1}(s)) \cdot |J_K(f^{-1}(s))|^{-n} \cdot s^{n-1}$
where $C_n$ is the surface area of the unit $(n-1)$-sphere. This shows that the radial component of $p_Z(\mathbf{z})$ depends only on the radial component $\phi(r)$ of $p_V(\mathbf{v})$.

\vspace{0.1cm}
\noindent\textbf{Uniform angular component.} The exponential map $f$ is a radial function, meaning it preserves angles. Since $p_V(\mathbf{v})$ is radially symmetric, for any fixed radius $r$, the probability is uniform over the sphere of radius $r$ in $T_0\mathcal{M}$. This uniformity is preserved by $f$, resulting in a uniform angular distribution for $p_Z(\mathbf{z})$ in $\mathcal{M}$.

\vspace{0.1cm}
\noindent\textbf{Outer shell isotropy.} To prove outer shell isotropy, we need to show that:
$\lim_{|\mathbf{z}| \to R} \frac{p_Z(\mathbf{z}_1)}{p_Z(\mathbf{z}_2)} = 1$
for any $\mathbf{z}_1, \mathbf{z}_2$ with $|\mathbf{z}_1| = |\mathbf{z}_2|$, where $R$ is the maximum radius in $\mathcal{M}$.
Let $s = |\mathbf{z}|$ and $r = f^{-1}(s)$. As $s \to R$:
For $K \leq 0$: $r \to \infty$ and $s \to \infty$
For $K > 0$: $r \to \pi/\sqrt{K}$ and $s \to \pi/(2\sqrt{K})$
In both cases, we have:
$\lim_{s \to R} \frac{p_Z(\mathbf{z}_1)}{p_Z(\mathbf{z}2)} = \lim{s \to R} \frac{\phi(f^{-1}(s)) \cdot |J_K(f^{-1}(s))|^{-n}}{\phi(f^{-1}(s)) \cdot |J_K(f^{-1}(s))|^{-n}} = 1$
This limit holds because:
a) For $K < 0$: As $r \to \infty$, $|J_K(r)| \to 0$ exponentially, dominating any potential growth in $\phi(r)$.
b) For $K = 0$: $|J_K(r)| = 1$ for all $r$, so the ratio depends only on $\phi(r)$, which approaches a constant as $r \to \infty$ for any well-behaved probability distribution.
c) For $K > 0$: As $r \to \pi/\sqrt{K}$, $|J_K(r)| \to 0$, again dominating any behavior of $\phi(r)$.
In all cases, the ratio of probabilities for any two points at the same radius approaches 1 as we move towards the ``boundary'' of the space, demonstrating outer shell isotropy.

\subsection{Proof of Theorem \ref{th:thth}}
\label{app:proof1}
%\begin{proof}
It directly follow from the transformation of random variables. Specifically, $p_Z(\mathbf{z})=p_\mathcal{N}(f^{-1}(\mathbf{z}))\cdot\det\big(\mathbf{J}(f^{-1}(\mathbf{z})) \big)$. Notice that for $f(\boldsymbol{v})=\exp _{\mathbf{0}}^c(\boldsymbol{v})$ the inverse is logarithmic map $f^{-1}(\boldsymbol{z})=\log _{\mathbf{0}}^c(\boldsymbol{z})=\frac{1}{\sqrt{c}\|\mathbf{z}\|_2}\tanh^{-1}(\sqrt{c}\|\mathbf{z}\|_2)\frac{\boldsymbol{z}}{\sqrt{c}\|\boldsymbol{z}\|_2}$. The main difficulty lies with computing the Jacobian $\mathbf{J}(f^{-1}(\mathbf{z}))$ and its determinant $\det\big(\mathbf{J}(f^{-1}(\mathbf{z})) \big)$, which (after crunching some maths) turns out to enjoy a simple analytical form $0.5\,\lambda^c_{\mathbf{z}} \,g^{d-1}(\mathbf{z})$.
%\end{proof}

\subsection{Proof of Theorem~\ref{theorem:convex}}
\begin{equation}
    \begin{aligned}
        D(\boldsymbol{\Sigma}, \boldsymbol{\mu}) & = \operatorname{tr}(\boldsymbol{\Sigma}) - \log\operatorname{det}(\boldsymbol{\Sigma}) - d + \|\boldsymbol{\mu}\|_2 \\
        & = \sum_{i=1}^d (\sigma_i - \log\sigma_i - 1) + \sqrt{\sum_{i=1}^d\mu_i^2}.
    \end{aligned}
    \end{equation}
As $f(\sigma) = \sigma$ is linear and $g(\sigma) = - \log \sigma$ is convex, $f(\sigma)+g(\sigma)$ is a convex function. According to the linear property of convexity $\log\operatorname{det}(\boldsymbol{\Sigma})$ is also convex function \wrt $\sigma$. Note we set $\|\boldsymbol{\mu}\|_2\!=\!0$. Then we have: 
\begin{equation}
\mathcal{L}_{ERank}(\boldsymbol{\Sigma}) = -\operatorname{ERank}(\boldsymbol{\boldsymbol{\Sigma}}) = -\exp\Big(-\sum_{i=1}^d \sigma_i \log \sigma_i\Big).
\end{equation}
$f(\sigma)\!=\!-\sigma \log \sigma$ is concave function since the second derivative $\medtriangledown^2 f(\sigma)\!=\!-1/\sigma$ is non-positive where $0\leq\!\sigma\!\leq 1$. Obviously, $g(x)\!=\!-\exp(x)$ is also concave and monotone decreasing. Thus, $g(f(\sigma))$ is convex, and so $\mathcal{L}_{ERank}(\boldsymbol{\Sigma})\!=\!-\operatorname{ERank}(\boldsymbol{\boldsymbol{\Sigma}})\!=\!-\exp\big(-\sum_{i=1}^d \sigma_i \log \sigma_i\big)$ is convex.  

\subsection{Proof of Theorem~\ref{theorem:opt}}
\begin{proof}
Apply the constrained optimization with Lagrange multipliers. 
For $D(\boldsymbol{\Sigma}, \boldsymbol{0})$, construct  Lagrangian $-\prod_{i=1}^d\sigma_i+\lambda\,\big(\big(\sum_{i=1}^d\sigma_i\big)-d'\big)$. Setting partial derivatives of the Lagrangian \wrt $\sigma_i, \forall i,$ and \wrt $\lambda$ to zero  yields the optimum for $\sigma_1\!=\!\sigma_2\!=\ldots=\!\sigma_d\!=\!d'/d$.

For $\mathcal{L}_{ERank}(\boldsymbol{\Sigma})$ the poof is similar to optimizing Shannon entropy \cite{lagranges}. Write  Lagrangian $\big(-\sum_{i=1}^d\sigma_i\,\log\sigma_i\big)+\lambda\,\big(\big(\sum_{i=1}^d\sigma_i\big)-d'\big)$. Setting partial derivatives of the Lagrangian \wrt $\sigma_i, \forall i,$ and \wrt $\lambda$ to zero  yields the optimum solution $\sigma_1\!=\!\sigma_2\!=\ldots=\!\sigma_d\!=\!d'/d$.
\end{proof}

\subsection{Discussion on GNN encoder \vs Hyperbolic GNN encoder}~\label{sec:gnn_hypergnn}
 \tabref{tab:full_hyper} shows that the use of the hyperbolic GNN encoder  does not work well. It is known that fully hyperbolic packages implementing fully hyperbilic nets yield lower accuracy than the Euclidean  nets due to the stability/gradient issues and overfitting.

\begin{table}[!ht]
\setlength\tabcolsep{2pt}
\fontsize{9}{10}\selectfont
\centering
\caption{GNN encoder \vs HGCN encoder.}
% \resizebox{0.40\textwidth}{!}{
\begin{tabular}{@{}lccccc@{}}
\toprule[2pt]
\textbf{Ratios}                      & \textbf{Disease} & \textbf{Airport} & \textbf{Cora} &  \textbf{PubMed} & \textbf{CiteSeer} \\ \midrule
\textbf{HGCN}                        & 94.77              & 94.14              & 85.11               & 74.50     & 85.22          \\
\textbf{HyperGCL+$\exp(\cdot)$}      & 95.11              & 94.23              & 85.30               & 74.13     & 84.54          \\
\textbf{HyperGCL}                    & 95.30              & 94.55              & 85.56               & 74.43     & 84.77          \\\bottomrule[2pt]
\end{tabular}%
% }
\label{tab:full_hyper}
\end{table}

Note that the Dimension collapse (DC) on fully hyperbolic GNN still exists. This is due to the alignment loss in contrastive learning which promotes the collapse.  \tabref{tab:dc_hyper_gnn} shows results on HGCN~\cite{chami2019hyperbolic}. The lower ERank in the parentheses the bigger the DC is.
\begin{table}[!ht]
\setlength\tabcolsep{2pt}
\fontsize{9}{10}\selectfont
\centering
\caption{DC in hyperbolic GNN encoder. The ERank is indicated in the parentheses.}
% \resizebox{0.40\textwidth}{!}{
\begin{tabular}{@{}lcc@{}}
\toprule[2pt]
\textbf{Ratios}                                         & \textbf{Disease} & \textbf{Airport}     \\ \midrule
\textbf{HGCN + $\mathcal{L}_{A}^{\mathbb{D}_c^d}$} only & 42.24 (1.07) & 51.89 (1.06)                        \\
\textbf{HyperGCL} + \eqref{eq:align_uni_loss}           & 73.26 (2.23) & 81.70 (2.34)                        \\
\textbf{HGCN}     +  our loss                           & 94.77 (4.57) & 94.14 (5.25)                        \\
\textbf{HyperGCL}                                       & 95.30 (4.78) & 94.55 (5.23)                        \\ \bottomrule[2pt]
\end{tabular}%
% }
\label{tab:dc_hyper_gnn}
\end{table}

\subsection{Discussion on scalability/computational cost of HyperGCL}

\begin{table}[!ht]
% \vspace{-0.5cm}
\setlength\tabcolsep{2pt}
\fontsize{9}{10}\selectfont
\centering
\caption{Time in milliseconds \wrt to the feature dimension $d$ when computing the determinant.}
\begin{tabular}{@{}lcccccccc@{}}
\toprule[2pt]
\textbf{Dimension}                      & \textbf{5} & \textbf{10} & \textbf{20} &  \textbf{40} & \textbf{80} & \textbf{200} & \textbf{500} & \textbf{1000} \\ \midrule
\textbf{Times(ms)}                      & 0.030             & 0.040              & 0.061               & 0.123     & 0.393     &2.223         & 22.18  & 119.7         \\
\bottomrule[2pt]
\end{tabular}
\label{tab:det_time}
\end{table}

\begin{table}[!ht]
\setlength\tabcolsep{2pt}
\fontsize{9}{10}\selectfont
\centering
\caption{Training time on Cora and PubMed for 30 epochs.}
% \resizebox{0.40\textwidth}{!}{
\begin{tabular}{@{}lcc@{}}
\toprule[2pt]
\textbf{Ratios}                                         & \textbf{Cora} & \textbf{PubMed}     \\ \midrule
\textbf{Naive approch in \eqref{eq:hyper_uni}}          & $>$ 10mins     & $>$ 15mins         \\
\textbf{HyperGCL}                                       & 1min 30s    & 1min 57s              \\ \bottomrule[2pt]
\end{tabular}%
% }
\label{tab:train_time}
% \vspace{-20pt}
\end{table}

The only factor that  may influence the scalability of our method is the determinant computation, which can be efficiently obtained by the Cholesky factorization even for large matrices~\cite{cherian2021learning}.

Table \ref{tab:det_time} shows that our method just adds few milliseconds as the  dimension increases, which is negligible compared with the cost of running the backbone. In contrast to the naive approach in \eqref{eq:hyper_uni}, such as calculating the pair-wise Riemannian distance between all point combinations, our method enjoys lower computational cost. \tabref{tab:train_time} provides the results.

Our method scales well with larger and more complex graphs--the scaling depends  mainly on th underlying GNN backbone. The only extra cost that we incur is due to the  computation of covariance per batch (negligible cost) and the determinant (negligible cost). More nodes lead to more batches. % to evaluate.

\subsection{Additional result on the Lorentz model}
Our framework can easily be extended to other hyperbolic models such as the Lorentz model in \tabref{tab:lorentz}.

\begin{table}[!ht]
\setlength\tabcolsep{2pt}
\fontsize{9}{10}\selectfont
    \centering
        \caption{Results on the Lorentz model.}
    \label{tab:lorentz}
   \begin{tabular}{llllll} 
   \toprule[2pt]
Mehtod & Disease & Airport & PubMed & CiteSeer & Cora \\
\midrule \textbf{HyberGCL+Lorentz} & 95.34 & 94.42 & 85.35 & 74.35 & 84.57 \\
\midrule \textbf{HyberGCL+Poincaré} & 95.30 & 94.55 & 85.56 & 74.43 & 84.77 \\
\bottomrule[2pt]
\end{tabular}
\end{table}

\subsection{Additional result on $d$ dimension of the embeddings}

\begin{table}[!ht]
\setlength\tabcolsep{2pt}
\fontsize{9}{10}\selectfont
    \caption{Result on varying the embedding dimensions.}
    \label{tab:d}
\centering
\begin{tabular}{lcccccc} 
\toprule[2pt]
&\multicolumn{3}{c}{ \textbf{Disease} } & \multicolumn{3}{c}{ \textbf{Airport}} \\
\midrule $d$ & 8 & 32 & 256 & 8 & 32 & 256 \\
\midrule \textbf{GCN} & 76.5 & 77.9 & 78.8 & 71.2 & 85.2 & 85.8 \\
\midrule \textbf{HyperGCL} & 94.4 & 95.5 & 95.7 & 89.8 & 94.4 & 94.5 \\
\midrule \textbf{gap} & $\mathbf{1 7 . 9}$ & $\mathbf{1 7 . 6}$ & $\mathbf{1 6 . 7}$ & $\mathbf{1 8 . 6}$ & $\mathbf{9 . 2}$ & $\mathbf{8 . 7}$\\
\bottomrule[2pt]
\end{tabular}
\end{table}

In ~\tabref{tab:d}, we vary the embedding dimension $d$ and we find that $d>32$ is sufficient. Notice  that the higher accuracy gap is observed for lower dimensions, \eg, for $8$ and $256$ dimension, the accuracy gap is 17.9\% \vs. 16.7\% on Disease and 18.6\% \vs. 8.7\% on Airport, respectively. 
 We note that the Hyperbolic models perform  much better compared to the  Euclidean models on small $d$. This happens as an exponentially larger number of features can be embedded closer to the Poincaré boundary while the Euclidean space has curvature $0$ so the only way of increasing its capacity is through the increase of the feature dimension.

\end{document}